\documentclass[12pt]{article}

\usepackage{url}            %
\usepackage{booktabs}       %
\usepackage{amsfonts}       %
\usepackage{nicefrac}       %
\usepackage{microtype}      %
\usepackage{xcolor}         %
\usepackage{mathtools}
\mathtoolsset{showonlyrefs}
\usepackage{bm}
\usepackage{bbm}
\usepackage{tocloft}
\usepackage{enumerate}
\usepackage{footnote}
\usepackage{amssymb}
\usepackage{amsmath}
\usepackage{mathrsfs}
\usepackage{caption}
\usepackage{subcaption}
\usepackage{dsfont}
\usepackage{adjustbox}
\usepackage{array}
\usepackage{colortbl}
\usepackage{graphicx}
\usepackage{multirow, hhline}
\usepackage{natbib}
\usepackage[ruled,linesnumbered]{algorithm2e}

\usepackage{hyperref}
\hypersetup{ hidelinks }

\allowdisplaybreaks
\usepackage{amsmath}
\makeatletter
\let\save@mathaccent\mathaccent
\newcommand*\if@single[3]{%
  \setbox0\hbox{${\mathaccent"0362{#1}}^H$}%
  \setbox2\hbox{${\mathaccent"0362{\kern0pt#1}}^H$}%
  \ifdim\ht0=\ht2 #3\else #2\fi
  }
\newcommand*\rel@kern[1]{\kern#1\dimexpr\macc@kerna}
\newcommand*\widebar[1]{\@ifnextchar^{{\wide@bar{#1}{0}}}{\wide@bar{#1}{1}}}
\newcommand*\wide@bar[2]{\if@single{#1}{\wide@bar@{#1}{#2}{1}}{\wide@bar@{#1}{#2}{2}}}
\newcommand*\wide@bar@[3]{%
  \begingroup
  \def\mathaccent##1##2{%
    \let\mathaccent\save@mathaccent
    \if#32 \let\macc@nucleus\first@char \fi
    \setbox\z@\hbox{$\macc@style{\macc@nucleus}_{}$}%
    \setbox\tw@\hbox{$\macc@style{\macc@nucleus}{}_{}$}%
    \dimen@\wd\tw@
    \advance\dimen@-\wd\z@
    \divide\dimen@ 3
    \@tempdima\wd\tw@
    \advance\@tempdima-\scriptspace
    \divide\@tempdima 10
    \advance\dimen@-\@tempdima
    \ifdim\dimen@>\z@ \dimen@0pt\fi
    \rel@kern{0.6}\kern-\dimen@
    \if#31
      \overline{\rel@kern{-0.6}\kern\dimen@\macc@nucleus\rel@kern{0.4}\kern\dimen@}%
      \advance\dimen@0.4\dimexpr\macc@kerna
      \let\final@kern#2%
      \ifdim\dimen@<\z@ \let\final@kern1\fi
      \if\final@kern1 \kern-\dimen@\fi
    \else
      \overline{\rel@kern{-0.6}\kern\dimen@#1}%
    \fi
  }%
  \macc@depth\@ne
  \let\math@bgroup\@empty \let\math@egroup\macc@set@skewchar
  \mathsurround\z@ \frozen@everymath{\mathgroup\macc@group\relax}%
  \macc@set@skewchar\relax
  \let\mathaccentV\macc@nested@a
  \if#31
    \macc@nested@a\relax111{#1}%
  \else
    \def\gobble@till@marker##1\endmarker{}%
    \futurelet\first@char\gobble@till@marker#1\endmarker
    \ifcat\noexpand\first@char A\else
      \def\first@char{}%
    \fi
    \macc@nested@a\relax111{\first@char}%
  \fi
  \endgroup
}
\makeatother

\newcommand{\tP}{\mathbb{P}}
\newcommand{\tE}{\mathbb{E}}

\newcommand{\bdeltaks}[1]{\bm{\delta}^{(#1)*}}

\newcommand{\btheta}{\bm{\theta}}

\newcommand{\bthetaks}[1]{\bm{\theta}^{(#1)*}}

\newcommand{\hthetak}[1]{\widehat{\bm{\theta}}^{(#1)}}
\newcommand{\bbeta}{\bm{\beta}}
\newcommand{\bbetaks}[1]{\bm{\beta}^{(#1)*}}

\newcommand{\hbetak}[1]{\widehat{\bm{\beta}}^{(#1)}}

\newcommand{\bA}{\bm{A}}
\newcommand{\bB}{\bm{B}}

\newcommand{\hbarA}{\widehat{\widebar{\bm{A}}}}

\newcommand{\bAks}[1]{\bm{A}^{(#1)*}}

\newcommand{\oA}{\widebar{\bm{A}}}

\newcommand{\Abarhat}{\widehat{\widebar{\bm{A}}}}

\newcommand{\thetahatk}[1]{\widehat{\bm{\theta}}^{(#1)}}

\newcommand{\thetahatt}{\widehat{\bm{\theta}}^{(t)}}

\newcommand{\Abarhatthetahatt}{\Abarhat \thetahatk{t}}

\newcommand{\bx}{\bm{x}}

\newcommand{\bxk}[1]{\bm{x}^{(#1)}}
\newcommand{\yk}[1]{y^{(#1)}}
\newcommand{\bXk}[1]{\bm{X}^{(#1)}}
\newcommand{\bYk}[1]{\bm{Y}^{(#1)}}

\newcommand{\bSigmak}[1]{\bm{\Sigma}^{(#1)}}

\newcommand{\hSigmak}[1]{\widehat{\bm{\Sigma}}^{(#1)}}
\newcommand{\fk}[1]{f^{(#1)}}

\newcommand{\epsilonk}[1]{\epsilon^{(#1)}}
\newcommand{\mO}{\mathcal{O}^{p \times r}}

\newcommand{\mQ}{\mathbb{Q}}

\newcommand{\sigmamin}{\sigma_{\min}}

\newcommand{\sigmamax}{\sigma_{\max}}

\newcommand{\wpr}{w.p. at least $1-e^{-C(r+\log T)}$}
\newcommand{\wpp}{w.p. at least $1-e^{-C(p+\log T)}$}
\newcommand{\wppp}{w.p. at least $1-e^{-C'(p+\log T)}$}

\newcommand{\barzeta}{\bar{\zeta}}
\newcommand{\zetak}[1]{\zeta^{(#1)}}
\newcommand{\zetabar}{\bar{\zeta}}

\makeatletter
\newcommand*{\Rom}[1]{\expandafter\@slowromancap\romannumeral #1@}
\makeatother
\newcommand*{\rom}[1]{\romannumeral #1}

\DeclareMathOperator*{\argmin}{arg\,min}

\newcommand{\twonorm}[1]{\|#1\|_{2}}

\usepackage[margin=1in]{geometry}
\usepackage{amsmath, amssymb, amsthm}
\usepackage{mathtools}
\usepackage{hyperref}

\newtheorem{theorem}{Theorem}
\newtheorem{lemma}[theorem]{Lemma}
\newtheorem{corollary}[theorem]{Corollary}
\newtheorem{proposition}[theorem]{Proposition}
\newtheorem{assumption}{Assumption}
\newtheorem{definition}{Definition}
\newtheorem{remark}[theorem]{Remark}

\newcommand{\R}{\mathbb{R}}
\newcommand{\E}{\mathbb{E}}

\newcommand{\norm}[1]{\left\lVert#1\right\rVert}
\newcommand{\opnorm}[1]{\left\lVert#1\right\rVert_{\mathrm{op}}}
\newcommand{\fnorm}[1]{\left\lVert#1\right\rVert_{\mathrm{F}}}

\newcommand{\matr}[1]{\bm{#1}}

\newcommand{\Aout}{\matr{A}_{\text{out}}}

\newcommand{\Bhatst}{\widehat{\bm{B}}_{\text{st}}}

\usepackage{xcolor} %
\definecolor{deepred}{RGB}{139, 0, 0} %

\usepackage{mathtools}
\DeclarePairedDelimiter{\card}{\lvert}{\rvert}

\newcommand{\bN}{\bm{N}}
\newcommand{\bP}{\bm{P}}

\newcommand{\sigmaminin}{\sigma_{\min,\text{in}}}

\newcommand{\bBst}{\widehat{\bB}_{\textup{st}}}

\newcommand{\st}{\textup{st}}

\newcommand{\bD}{\bm{D}}

\newcommand{\abs}[1]{\left\lvert #1 \right\rvert} %

\newcommand{\cardS}{\card{S}}

\usepackage{graphicx}   %
\usepackage{caption}    %
\usepackage{subcaption} %
\graphicspath{{figure/}} %

\usepackage{authblk}

\title{Robust and Adaptive Spectral Method for Representation Multi-Task Learning with Contamination}

\author[1]{Yian Huang\thanks{\texttt{huang.yian@columbia.edu}}}
\author[2]{Yang Feng\thanks{\texttt{yang.feng@nyu.edu}}}
\author[1]{Zhiliang Ying\thanks{\texttt{zying@stat.columbia.edu}}}

\affil[1]{Department of Statistics, Columbia University}
\affil[2]{Department of Biostatistics, School of Global Public Health, New York University}

\date{\today}

\begin{document}

\maketitle

\begin{abstract}
Representation-based multi-task learning (MTL) improves efficiency by learning a shared structure across tasks, but its practical application is often hindered by contamination, outliers, or adversarial tasks. Most existing methods and theories assume a clean or near-clean setting, failing when contamination is significant. This paper tackles representation MTL with an unknown and potentially large contamination proportion, while also allowing for heterogeneity among inlier tasks. 
We introduce a Robust and Adaptive Spectral method (RAS) that can distill the shared inlier representation effectively and efficiently, while requiring no prior knowledge of the contamination level or the true representation dimension. Theoretically, we provide non-asymptotic error bounds for both the learned representation and the per-task parameters. These bounds adapt to inlier task similarity and outlier structure, and guarantee that RAS performs at least as well as single-task learning, thus preventing negative transfer. We also extend our framework to transfer learning with corresponding theoretical guarantees for the target task. Extensive experiments confirm our theory, showcasing the robustness and adaptivity of RAS, and its superior performance in regimes with up to 80\% task contamination.
\end{abstract}

\section{Introduction}

Multi–task learning (MTL) improves statistical efficiency and predictive accuracy by sharing information across related tasks, instead of learning each task in isolation 
\citep{Caruana1997,Baxter2000,EvgeniouPontil2004,AndoZhang2005}. Classical MTL couples task predictors via shared regularizers, task clustering, or learned task relationships \citep{Argyriou2008,TraceNormMTL2009,JacobBachVert2008,KumarDaume2012,XueLiaoCarinKrishnapuram2007}. In parallel, {representation learning} provides transferable, low-dimensional structure that can dramatically reduce sample complexity across tasks \citep{BengioCourvilleVincent2013,LeCunBengioHinton2015}. These threads meet in {representation-based MTL}, which assumes task parameters lie approximately in a shared low-rank subspace, enabling joint estimation of a common representation and per–task coefficients \citep{Argyriou2008,maurer2016benefit,ZhangYang2017,Crawshaw2020}.

Despite the provable gains over single-task learning delivered by representation-based MTL 
under ideal scenarios
\citep{maurer2016benefit,tripuraneni2020theory,thekumparampil2021statistically,Saunshi2021}, in practice, however, modern pipelines aggregate large, heterogeneous task collections harvested automatically. 
As a result, outlier tasks, contaminated tasks, or even tasks subject to adversary attack may account for non-trivial proportion among all tasks in the MTL.\footnote{In this paper, we do not distinguish between outlier tasks, contaminated tasks and adversarily attacked tasks, given their similarity. We shall use these words interchangeably in the paper.}

For example, in the application of 
computer vision and autonomous driving,  multi-task networks jointly predict depth, surface, and instance segmentation from shared backbones \citep{SenerKoltun2018,Kendall2018Uncertainty,Standley2020Which}. At scale, tasks harvested from different cities, sensors, or weather regimes can be weakly related, corrupted (e.g., due to miscalibrated LiDAR, systematic label noise, or even adversarially attacked).
Such {task-level} anomalies act as outliers that can distort the shared representation if not robustly handled.
In genomics and regulatory modeling, multi-assay prediction across cell types/assays naturally forms an MTL setting (e.g., multi-label regulatory prediction) \citep{ZhouTroyanskaya2015DeepSEA,Kelley2016Basset}. Cross-lab batch effects and mislabeled or low-quality cohorts can yield contaminated tasks \citep{Johnson2007ComBat,Leek2010SVA}. A robust shared representation is essential to avoid negative transfer from corrupted assays or atypical cell types.
In healthcare scenarios, survival models are often trained across multiple centers with heterogeneous coding practices, small cohorts, or population shifts \citep{Ranganath2016DeepSurvival,Katzman2018DeepSurv,AlaaSchaar2018DeepHit}. Erroneous or atypical sites behave like outlier or contaminated tasks. Robust representation learning is needed to stabilize hazard estimation while preserving site-specific adaptations.

These practical scenarios raises persistent challenges and requirements to  guard against {negative transfer} when some tasks are weakly related or harmful \citep{PanYang2010Survey,weiss2016survey,Rosenstein2005Negative,Standley2020Which}.
In the literature, however, robustness to {contaminated  tasks} that violate the shared-structure assumption is less explored. 
\cite{du2020few,tripuraneni2020theory,tripuraneni2021provable,niu2024collaborative,thekumparampil2021statistically} assume the absence of contaminated tasks in the MTL, with  exactly the same shared representation. 
\cite{chua2021fine,duan2023adaptive} follow the assumption of absence of outliers, while allowing similar but not exactly the same  shared representation among tasks. 
\cite{TianGuFeng2023LearningFromSimilar} relaxes the stringent assumption about outliers, proposing an algorithm that allows some small contamination proportion. 
Some sparsity and decomposition formulations are considered with the aim of detecting and isolating outlier tasks while sharing features among inliers \citep{jalali2010dirty, gong2012robust}, but they require solving computationally complicated optimization, and only tolerate some small proportion of outliers.

In this paper, we study 
representation MTL with unknown  and potentially large contamination proportion where the contaminated data can follow arbitrary distribution, while allowing heterogeneity of representations among tasks.
We propose a robust and adaptive spectral method (RAS).
Specifically,
RAS adopts 
a {data-driven singular-value threshold} calibrated to the perturbation level, retaining directions whose signal exceeds the heterogeneity and noise floor. A final biased regularization step anchors task estimates to the learned subspace while preserving per–task adaptivity. 
RAS  requires neither the knowledge of  intrinsic dimension nor the contamination proportion, sufficiently exploits the approximate low-rank structure of the outliers if available, and also gracefully
adapts to general-rank outlier structure. 
Besides, RAS remains tuning-light and computationally efficient, facilitating its applications in practical scenarios. 
We derive the theoretical guarantees about the representation estimation error, and per-task coefficient estimation error. In particular, our non-asymptotic error bounds adapt to the inlier similarity and outlier structure, 
which facilitates its 
superior performance
in the existence of large contamination proportion. Also, RAS avoids the negative transfer, with a safe fallback matching the single-task error rates.
We also demonstrate the application of RAS in the transfer learning (TL), and show the theoretical guarantees for the estimation error of a target task.
Besides, we  conduct extensive numerical experiments that corroborate our theoretical findings. In particular, the numerical experiments test contamination proportion in a large regime, with up to $80\%$ contaminated tasks, to showcase the robustness and adaptivity of RAS, and its  superior performance compared with other methods that require the information about oracle intrinsic dimension and contamination proportion.

\subsection{Related Literature}

MTL improves sample efficiency by exploiting relatedness among tasks \citep{Caruana1997,Baxter2000}. Early regularization or kernal approaches couple tasks through shared regularizers or output kernels \citep{EvgeniouPontil2004,EvgeniouMicchelliPontil2005,Argyriou2008}, group feature sharing \citep{Lounici2011}, and task-relationship or clustering structure to ``share with the right neighbors" \citep{JacobBachVert2008,KangGraumanSha2011,ZhouChenYe2011,BarzilaiCrammer2015}. Surveys synthesize benefits and risks, especially {negative transfer} when relatedness assumptions fail \citep{ZhangYang2017,Ruder2017}.

A central viewpoint posits that task parameters lie approximately in a low-dimensional shared representation. Convex trace-norm models \citep{Argyriou2008} and structural learning from multiple tasks (often with unlabeled data) \citep{AndoZhang2005} formalized this intuition. Generalization theory quantifies when shared representations beat single-task baselines as a function of intrinsic dimension, number of tasks, and per-task samples \citep{maurer2016benefit}. 
Recent representation MTL or meta-learning literature analyze learning under exact or near-shared subspaces but typically assume no outlier tasks or only small contamination proportion \citep{du2020few,tripuraneni2020theory,tripuraneni2021provable,thekumparampil2021statistically,TianGuFeng2023LearningFromSimilar,duan2023adaptive,niu2024collaborative}. \citet{gu2024robust} considers a special case of $1$-dim representation.

Robust losses and estimators 
control the effect of heavy tails and corrupted labels within a task and can be integrated into multi-task objectives \citep{Huber1964,Catoni2012,Minsker2015,LugosiMendelson2019}. These techniques improve per-task estimation without directly addressing task-level outliers that violate shared-representation assumptions.

When entire tasks are contaminated or weakly related, the shared representation can be corrupted. One line decomposes the coefficient matrix into a shared low-rank component plus sparse task-specific deviations 
\citep{jalali2010dirty}, which can flag anomalous tasks but typically assumes small contamination proportion and requires heavier optimization. A related line comes from robust PCA, which separates low-rank signal from sparse outliers with convex and  nonconvex algorithms \citep{CandesRPCA2011,xu2012robust,Netrapalli2014}. These results clarify when a shared subspace can be recovered under small corruptions.

When tasks are domains, theory for multi-source domain adaptation and domain generalization clarifies when combining sources helps a target and why invariant representations matter \citep{MansourMohriRostamizadeh2009,BenDavid2010,Muandet2013}. In decentralized settings, federated MTL frameworks support personalized yet shared structure under heterogeneity \citep{Smith2017}, while Byzantine-robust aggregation 
protects against adversarial or corrupted clients \citep{BlanchardEtAl2017,YinChenKannanRamchandran2018}. These threads are complementary to representation robustness: they mitigate {where} gradients/updates come from, not {what} representation is estimated.

\subsection{Notations and Organization}

Throughout the paper we adopt the following conventions. Probabilities and expectations are written as $\tP$ and $\tE$. For positive sequences $\{a_n\}$ and $\{b_n\}$, we write $a_n=o(b_n)$ (or $a_n\ll b_n)$ when $a_n/b_n\to 0$; $a_n=O(b_n)$ (or $a_n\lesssim b_n$) when there exists a universal constant $C<\infty$ with $a_n/b_n\le C$; and $a_n\asymp b_n$ when both $a_n/b_n$ and $b_n/a_n$ are bounded by some universal constant $C<\infty$. 
For a random variable $x_n$ and positive numbers $a_n$, the relation $x_n=\mathcal{O}_{\tP}(a_n)$ means: for every $\epsilon>0$ there exists $M<\infty$ such that $\sup_n \tP(|x_n|> M a_n)\le \epsilon$.
For a vector $\bx\in\mathbb{R}^d$, we denote its Euclidean norm by $\twonorm{\bx}$. For scalars $a,b$, set $a\vee b=\max\{a,b\}$ and $a\wedge b=\min\{a,b\}$. For any $K\in\mathbb{N}$, we use $[K]=\{1,\dots,K\}$ and $1\!:\!K$ for the standard index set. If $S$ is a set, then $|S|$ denotes its cardinality and $S^{\mathrm{c}}$ its complement. “w.p.” means “with probability.” Absolute constants $c, c', c'', c_1, c_2$ and $C,C',C'',C_1,C_2$ may change from line to line.

The rest of the paper is organized as follows.
Section~\ref{sec:setup-and-algo} presents the problem setup and the RAS algorithm. Section~\ref{sec:theory} states our non-asymptotic theoretical guarantees.
Section~\ref{sec:TL} demonstrates the application of our RAS algorithm in the transfer learning, and its theoretical ganrantees.
Section~\ref{sec:experiment} reports empirical results, followed by Section~\ref{sec:discussion} that concludes the paper.
All the proofs and supplementary lemmas are deferred to the Appendix.

\section{Problem Setup and Algorithm}
\label{sec:setup-and-algo}

\subsection{Problem Setup}\label{subsec: setup}

In this section, we formalize the MTL problem setup investigated in the paper. Consider $T$ supervised regression tasks. For task $t \in [T]$, we observe i.i.d. samples $\{(\bxk{t}_i,\yk{t}_i)\}_{i=1}^n$ with covariates $\bxk{t}_i \in \mathbb{R}^p$ and responses $\yk{t}_i \in \mathbb{R}$.
The $T$ tasks are partitioned into $S$ and its complement $S^c$. Tasks in $S$ are called {inlier} tasks, which share closely related representations, whereas tasks in $S^c$ serve as {outlier}, {contaminated} tasks, or adversarially attacked tasks,
for which we impose no distributional structure. 
Let $\varepsilon=\frac{\card{S^c}}{T}$ be the contamination proportion. We allow $\varepsilon$ to be very large, with the trivial requirement that $\varepsilon < 1.$
Note that the inlier subset $S \subseteq [T]$ is unknown. 

For each $t \in S$, a linear model is assumed:
\begin{equation}\label{eq:linear model}
  \yk{t}_i = (\bxk{t}_i)^\top \bbetaks{t} + \epsilonk{t}_i,\qquad i=1,\ldots,n,
\end{equation}
where $\bbetaks{t} = \bAks{t}\bthetaks{t}$ for some $\bAks{t}\in \mO := \{\bA\in\mathbb{R}^{p\times r}:\bA^\top\bA=\bm I_r\}$ and a low-dimensional parameter $\bthetaks{t}\in\mathbb{R}^r$ with $r\le p$. The noise variables $\{\epsilonk{t}_i\}_{i=1}^n$ are mean-zero, sub-Gaussian, and independent of $\{\bxk{t}_i\}_{i=1}^n$.
Without loss of generality, 
take $\bxk{t}$ to be mean-zero with covariance $\bSigmak{t} := \mathbb{E}[\bxk{t} (\bm{x}^{(t)})^{\top}]$. 
Let 
$\mQ_{S^c}$ denote the joint law of $\{\{(\bxk{t}_i,\yk{t}_i)\}_{i=1}^n\}_{t\in S^c}$. 
Note that no structural assumptions are imposed on outlier tasks: they can follow any distribution.
For notational convenience, define the coefficient matrix $\bB_S^* \in \mathbb{R}^{p\times |S|}$ whose $t$-th column is $\bbetaks{t}$ for $t\in S$. Write
\[
\zetak{t} := \twonorm{\bthetaks{t}} = \twonorm{\bbetaks{t}}, 
\qquad 
\bar\zeta^2 := 
|S|^{-1} \sum_{t\in S} (\zetak{t})^2,
\]
and assume the signal-noise-ratio satisfies $\min_{t\in S}\zetak{t}\gtrsim \sqrt{(p+\log T)/n}$.

We quantify representation similarity among $\{\bAks{t}\}_{t\in S}$ using principal angles between the corresponding column spaces. Specifically, we assume there exists $h\in[0,1]$ such that
\begin{equation}\label{eq:A similarity}
  \min_{\oA\in \mO}\ \max_{t\in S}\ \twonorm{\bAks{t}(\bAks{t})^\top - \oA\oA^\top} \le h.
\end{equation}
Smaller $h$ indicates greater similarity; $h=0$ recovers the ``shared representation’’ setting studied in \citet{du2020few,tripuraneni2021provable}. The spectral norm of the projector difference in~\eqref{eq:A similarity} corresponds to the largest principal angle between the two subspaces \citep{wedin1972perturbation,chenchifan2021spectralmethodsfor}.

We make the following standard assumptions.

\begin{assumption}
\label{assump:covariate}
For any task $t \in S$, the feature vectors $\bm{x}_i^{(t)}$ are i.i.d. sub-Gaussian random vectors.
The covariance matrix is $\matr{\Sigma}^{(t)} = \E[\bm{x}_i^{(t)}(\bm{x}_i^{(t)})^\top]$, and we assume $0 < c_{\min} \le \lambda_{\min}(\matr{\Sigma}^{(t)}) \le \lambda_{\max}(\matr{\Sigma}^{(t)}) \le c_{\max} < \infty$ for all $t \in S$. 
\end{assumption}

\begin{assumption}
\label{assump:diversity}
Let $\bm{\Theta}^*_S = [\bm{\theta}^{(t){*}}]_{t \in S} $.
There exist 
$\{\bthetaks{t} \}_{t \in S}$ such that
$\sigma_r(\bm{\Theta}^*_S/\sqrt{\card{S}}) \ge \sigma_{\min,\text{in}} > 0$.
\end{assumption}

\begin{assumption}\label{assump:n}
	$n \geq C(p+ \log T)$ with a  universal constant $C > 0$.
\end{assumption}

\begin{remark}
Note that Assumptions \ref{assump:covariate} and \ref{assump:n} are mild and widely adopted in the literature \citep{du2020few, duan2023adaptive, TianGuFeng2023LearningFromSimilar}.
Assumption \ref{assump:diversity} is often called the inlier task diversity assumption, and also widely adopted in the literature \citep{du2020few,chua2021fine,tripuraneni2021provable,duchi2022subspace,TianGuFeng2023LearningFromSimilar,niu2024collaborative}.
This assumption ensures that each direction of the shared latent representation is sufficiently explored and represented, 
which is necessary for the representation learning.
\end{remark}

Before presenting the algorithm, we introduce some notations.
Let
$\hbetak{t}_\st$ be the single-task coefficient estimate, and $\widehat{\bB}_\st = \{ \hbetak{t}_{\textup{st}} \}_{t \in T}.$
Let $\bB^*_S = \{ \bbetaks{t} \}_{t \in S}$, and 
$\widebar{\bB}_S=\{ \oA \oA^\top \bbetaks{t} \}_{t \in S}.$
Define a matrix 
$\widetilde{\bB}$ with its columns 
$\widetilde{\bB}_{:,t}=\oA \oA^\top \bbetaks{t}$
if $t \in S$, and $\widetilde{\bB}_{:,t}=
\hbetak{t}_{\st}$ if 
$t \in S^c.$
Denote the perturbation matrix $\bN =\Bhatst - \widetilde{\bB}$. 
Let 
$
\bdeltaks{t} = \oA^\perp(\oA^\perp)^\top
\bbetaks{t},
$
and 
$\bD_S=\{ \bdeltaks{t} \}_{t \in S}$.
Thus,
$\bN_S=(\Bhatst)_S - \bm{B}_S^*+ \bm{D}_S^*$.

\subsection{Our Algorithm: Robust and Adaptive Spectral Method}

\begin{algorithm}[!h]
\caption{Robust Adaptive Spectral Method (RAS)}
\label{algo:ras}
\KwIn{Data $\{\bXk{t}, \bYk{t}\}_{t=1}^T = \{\{\bxk{t}_i, \yk{t}_i\}_{i=1}^n\}_{t=1}^T$, penalty parameter $\gamma$, threshold $\tau$.
}
\KwOut{Estimators $\{\hbetak{t}\}_{t=1}^T, \hbarA$}
\underline{Step 1:} (Single-task regression) $\widehat{\bbeta}^{(t)}_{\text{st}} = \argmin_{\bbeta \in \mathbb{R}^p}\big\{ \fk{t}(\bbeta)\big\}$ for $t \in [T]$. Create a $p \times T$ matrix $\widehat{\bB}_{\text{st}}$ of which the $t$-th column is $\widehat{\bbeta}^{(t)}_{\text{st}}$ \\
\underline{Step 2:} (SVD) Conduct SVD $\frac{1}{\sqrt T} \widehat{\bB}_{\text{st}} = \widehat{\bm{U}}\widehat{\bm{\Lambda}}\widehat{\bm{V}}^\top$ with $\widehat{\bm{U}} \in \mathcal{O}^{p \times p}$.
Set $\hat{k} = \max\Big\{k \in [T]: \sigma_{k}(\widehat{\bB}_{\text{st}}/\sqrt{T}) \geq \tau$ \Big\}.
let $\widehat{\oA}$ be the first $\hat{k}$ columns of $\widehat{\bm{U}}$, and set $\hthetak{t} = \argmin_{\btheta \in \mathbb{R}^{\hat{k}}}\fk{t}(\widehat{\oA}\btheta)$ \\
\underline{Step 3:} (Biased regularization) $\hbetak{t} = \argmin_{\bbeta \in \mathbb{R}^p} \big\{\fk{t}(\bbeta) + \frac{\gamma}{\sqrt{n}}\twonorm{\bbeta - \widehat{\oA}\hthetak{t}}\big\}$ for $t \in [T]$
\end{algorithm}

In this section, we introduce our proposed algorithm: robust and adaptive spectral  (RAS) method. 
When there is no contamination, and each task is assumed to share a similar low dimensional representation as in \eqref{eq:A similarity}, singular value decomposition (SVD) serves as a natural option for estimation of the shared representation. 
However, when there is unknown and potentially large contamination proportion, and especially when the data of contaminated tasks can follow arbitrary distribution, the estimated subspace of SVD can be significantly impaired, especially given the vulnerability of spectral methods to outliers or contamination. 
However, we use an adaptive thresholding method, and show that the estimated subspace from our proposed RAS can still well capture the true inlier subspace, and thus lead to superior performance in the contaminated MTL problem, despite the existence of large contamination proportion, and arbitrary structure of contaminated tasks.

Our RAS method is demonstrated in Algorithm \ref{algo:ras}. Step 1 of RAS simply fits  single-task regressions, and forms the estimated coefficient matrix $\bBst$,
where  $\fk{t}(\bbeta)$ is the loss function. In the context of linear regression, $\fk{t}(\bbeta) 
= \frac{1}{2n}\sum_{i=1}^n [\yk{t}_i - (\bxk{t}_i)^\top\bbeta]^2$ for $\bbeta \in \mathbb{R}^p$.
The step 2 of RAS conducts SVD on the scaled estimated coefficient matrix $\frac{1}{\sqrt{T}} \bBst$, and uses the threshold $\tau$ to select the singular values and the corresponding singular vectors. Crucially, the threshold $\tau$ is set as $\tau \asymp \opnorm{\bm{N}}$, and in particular, for inlier tasks with linear model \eqref{eq:linear model}, it is set as:
$$
    \tau \asymp
    \sqrt{\frac{p+\cardS}{nT}}
    + h \barzeta \sqrt{1-\varepsilon}
    \bigg[\frac{\sigmamax(\bm{D}^*_S)}{\sqrt{r}\sigmamin(\bm{D}^*_S)} \wedge 1\bigg].
    $$
This step filters out the perturbation to $\widebar{\bB}_S$, but keeps the true signals.
Step 3 of RAS implements a regular biased regularization step, which uses task-specific data to fit per-task coefficient estimates, while adding a penalty term that involves the estimated coefficients using the shared representation from Step 2.
Such a regularization step is often used in the literature to prevent negative transfer \citep{scholkopf2001generalized,kuzborskij2013stability, duan2023adaptive,TianGuFeng2023LearningFromSimilar}.

\begin{remark}
    Note that our RAS method also applies to general inlier models beyond linear regression. And the threshold $\tau$ can still be set as $ \tau \asymp \opnorm{\bN}$, where the expression of $\opnorm{\bN}$ depends on the specific inlier model.
\end{remark}

\begin{remark}
    Our proposed RAS has major differences with the   spectral method (SM) proposed in \citet{TianGuFeng2023LearningFromSimilar}.
    SM assumes the knowledge of the true intrinsic dimension $r$, and uses the oracle $r$ in the algorithm and proof. 
    They proposed another algorithm \citep[Algorithm 3]{TianGuFeng2023LearningFromSimilar}  to recover the true $r$ with high probability under some assumptions \citep[Theorem 9]{TianGuFeng2023LearningFromSimilar}.
    
    In contrast, our RAS has a different design philosophy.
    With the presence of unknown and potentially large contamination proportion, the subspace estimate from the SVD may be severely ruined. In this case, even if the knowledge of the true intrinsic dimension $r$ is known, we choose not to use it. 
    Instead, we uses an adaptive thresholding method with a proper threshold to make sure the true inlier signals are mostly recovered 
    in the subspace  estimate, which makes the RAS method truly robust and adaptive to the unknown  and potentially large contamination proportion.

    Besides, the SM uses a winsorization step to alleviate the impact of outliers: if the $l_2$-norm of the coefficient estimate $\hbetak{t}$ is among the upper $\widebar{\varepsilon}$ of all $T$ tasks, it is winsorized.
    However,
    when the magnitude of the $\hbetak{t}$ for $t \in S^c$ is small, this winsorization procedure truncates the inlier coefficient estimates $\hbetak{t}$, and thus ruins the  inlier signal estimates,  as we will show in the  numerical experiments.
    Also, this step requires the knowledge of the true contamination proportion, or its upper bound, which, in most practical applications, is unknown.
    Our RAS does not require such knowledge, and is robust regardless of the magnitude of the $\hbetak{t}$ of contaminated tasks.
\end{remark}

\section{Theoretical Guarantees}
\label{sec:theory}
In this section, we present the theoretical guarantees for the proposed RAS method. 
We first introduce the notion of effective signal rank.

\begin{definition}[Effective Signal Rank]
Let $\widetilde{\lambda_j} = \sigma_j(\widetilde{\bB} / \sqrt{T})$ be the $j$-th singular value of the scaled matrix. Given the  threshold $\tau$, we define the  {effective signal rank} as:
\begin{equation}
    \widetilde{r}_{\text{eff}}
    = \max \{j \in [T] : 
    \widetilde{\lambda_j }
    > 1.25 \tau \}.
\end{equation}
\end{definition}

    The effective signal rank $\widetilde{r}_{\text{eff}}$ represents the number of signals of the {entire} task collection whose  strength is large enough to be reliably detected. In the proposition below, we show that under a mild spectrum gap condition, the estimated rank is equal to the effective signal rank with high probability.

\begin{proposition}
\label{thm:rank_general}
Let Assumption \ref{assump:covariate}-\ref{assump:n} hold. 
Let
$\tau = C'\left(
\sqrt{\frac{p+\cardS}{nT}}
    + h \barzeta \sqrt{1-\varepsilon}
\bigg[\frac{\sigmamax(\bm{D}^*_S)}{\sqrt{r}\sigmamin(\bm{D}^*_S)} \wedge 1\bigg] \right)$
with a sufficiently large constant 
$C'$.
Assume a gap in the true spectrum: 
$
\widetilde{\lambda}_{r_{\text{eff}}+1}
< 0.75 \tau$.
Then \wpp, the estimated rank from  RAS is
\begin{equation}
\hat{k} = \widetilde{r}_{\text{eff}}.
\end{equation}
\end{proposition}

\begin{remark}
    When the spectral gap assumption 
$
\widetilde{\lambda}_{r_{\text{eff}}+1}
< 0.75\tau$
    does not hold, we define 
     $\hat{r}_{\text{eff}} = \max\Big\{j \in [T]: \sigma_{j}(\widetilde{\bB}
     /\sqrt{T}) > 
     0.75
     \tau \Big\}.$
     By the same argument as in the proof of Proposition \ref{thm:rank_general}, we have 
     \wpp , 
     $\hat{k} \leq \hat{r}_{\text{eff}}$.
\end{remark}

In the lemma below, we show that when the matrix of single-task coefficient estimates of outliers 
is composed of a low rank matrix and some small perturbation, the estimated rank $\hat{k}$ is  small.  
\begin{lemma}
\label{lemma:low rank outlier}
    Suppose the matrix of single-task coefficient estimates of outliers has the decomposition: 
    $$ 
    (\widehat{\bB}_{\textup{st}})_{S^c}
    = \widebar{\bB}_{S^c} + \widetilde{\bN},
    $$
    where 
    $\textup{rank}(\widebar{\bB}_{S^c})=r_{\textup{out}}$, and
    $\opnorm{\widetilde{\bN}} \leq \frac{3}{4}\tau$ is the perturbation matrix.
    Let $r_{\cap} :=  \textup{dim}(\textup{col}(\widebar{\bB}_S) 
    \cap 
    \textup{col}(\widebar{\bB}_{S^c}) ),$
    which equals the number of zero principal angles between 
    $\widebar{\bB}_S$ and 
    $\widebar{\bB}_{S^c}. $
    Then \wpp, 
    \begin{equation}
    \hat{k} \leq 
    r+ r_{\textup{out}}- r_{\cap}.
    \end{equation}
\end{lemma}

    By Lemma \ref{lemma:low rank outlier}, if $r_{\textup{out}}-r_\cap \lesssim  r$, we have, $\hat{k} \lesssim r.$
    Besides, if the outliers share some basis with the inlier subspace $\oA$,  $\hat{k}$ will be smaller.
Note that in the worst case with adversarially contaminated $(\widehat{\bB}_{\textup{st}})_{S^c}$, we can choose $\widetilde{\bN}=\bm{0}$, and we still have $\hat{k} \leq  \min\{   r+ r_{\textup{out}}- r_{\cap}, r+T\varepsilon, p,T
\}.
$
Next, we show the subspace estimation error bound of the proposed RAS algorithm.

\begin{proposition}
\label{prop:subspace_error_bd_general}
    Suppose Assumptions \ref{assump:covariate}-\ref{assump:n} hold. 
    Let threshold 
    $$
    \tau \asymp
    \sqrt{\frac{p+\cardS}{nT}}
    + h \barzeta \sqrt{1-\varepsilon}
    \bigg[\frac{\sigmamax(\bm{D}^*_S)}{\sqrt{r}\sigmamin(\bm{D}^*_S)} \wedge 1\bigg]
    $$
    in the RAS algorithm.
    Then \wpp, 
    the following approximation error bound of the subspace estimate
    holds:
    \begin{equation}
    \label{eq:subspace_error1}
\opnorm{(\matr{I} - {\matr{P}}_{\Abarhat}) \oA} \lesssim 
\frac{1}
{\sigma_{\min,\text{in}} \sqrt{1-\varepsilon} } \sqrt{\frac{p+\cardS}{nT}}
+\frac{h\barzeta}
{\sigma_{\min,\text{in}}}    \bigg[\frac{\sigmamax(\bm{D}^*_S)}{\sqrt{r}\sigmamin(\bm{D}^*_S)} \wedge 1\bigg]
.
\end{equation}
\end{proposition}

\begin{corollary}
       Suppose Assumptions \ref{assump:covariate}-\ref{assump:n} hold.
        Let threshold 
    $$
    \tau \asymp
    \sqrt{\frac{p+\cardS}{nT}}
    + h \barzeta \sqrt{1-\varepsilon}
    \bigg[\frac{\sigmamax(\bm{D}^*_S)}{\sqrt{r}\sigmamin(\bm{D}^*_S)} \wedge 1\bigg]
    $$
    in the RAS algorithm.
       If we take the usual assumption that $\sigma_{\min,\text{in}} \geq \frac{c}{\sqrt r} \barzeta$
       where $c$ is a constant
       as in \citet{TianGuFeng2023LearningFromSimilar,du2020few, niu2024collaborative}, 
    the upper bound can be written as: 
\begin{equation}
\opnorm{(\matr{I} - {\matr{P}}_{\Abarhat}) \oA} \lesssim 
\frac{1}
{ \barzeta \sqrt{1-\varepsilon} } \sqrt{\frac{pr}{nT}}+
\frac{1}{\barzeta}  \sqrt{\frac{r}{n}}
+
h \sqrt r
    \bigg[\frac{\sigmamax(\bm{D}^*_S)}{\sqrt{r}\sigmamin(\bm{D}^*_S)} \wedge 1\bigg]
.
\end{equation}
\end{corollary}

    Under a mild assumption that the 
   contamination proportion $\varepsilon$ is upper bounded, e.g., $\varepsilon \leq 0.9$,
    the term $\frac{1}{\sqrt{1-\varepsilon}}$ 
    is upper bounded by a constant factor $3.2$, which does not affect the rate. 
    Next, we present the estimation error bound for the coefficient estimators returned by RAS.
    For notational convenience, define 
    $k^\star = \min\{   r+ r_{\textup{out}}, r+T\varepsilon, p,T
\}.$

\begin{theorem}
\label{thm:inlier_error}
Let Assumptions \ref{assump:covariate}-\ref{assump:n} hold.
Let threshold 
    $$
    \tau \asymp
    \sqrt{\frac{p+\cardS}{nT}}
    + h \barzeta \sqrt{1-\varepsilon}
    \bigg[\frac{\sigmamax(\bm{D}^*_S)}{\sqrt{r}\sigmamin(\bm{D}^*_S)} \wedge 1\bigg]
    $$
    in the RAS algorithm.
Let  
$\gamma =C' \sqrt{p+\log T}$ with a sufficiently large constant $C'$. 
Assume $\sigmaminin > 1.25\tau$.
Then \wpr, the $\ell_2$ estimation error for any inlier task $t \in S$ is bounded by:
\begin{align}
    \twonorm{\hbetak{t}-\bbetaks{t}}
    \lesssim 
    \bigg\{
    \sqrt{\frac{{k^\star}}{n}} + 
    \sqrt{\frac{\log T}{n}} + 
    h \zetak{t}
    +
    \frac{\zetak{t} }{\sigma_{\min,\text{in}} }
    &\left(
    \frac{1}{\sqrt{1-\varepsilon} }
    \sqrt{\frac{p}{nT}}
    + \sqrt{\frac{1}{n}}
    \right)\\
    &+
    \frac{\zetak{t}}{\sigmaminin}
    \barzeta h 
    \bigg[\frac{\sigmamax(\bm{D}^*_S)}{\sqrt{r}\sigmamin(\bm{D}^*_S)} \wedge 1\bigg]    \bigg\}\\
    &\wedge 
    \sqrt{\frac{p+\log T}{n}},
    \quad \forall t \in S.
\end{align}
And if the outlier tasks satisfy the linear model,
 w.p. at least $1-e^{-C'(p+\log T)}$, 
	\begin{equation}
		\max_{t \in S^c}\twonorm{\hbetak{t}-\bbetaks{t}} \lesssim \sqrt{\frac{p + \log T}{n}}.
	\end{equation}
\end{theorem}
\begin{corollary}
\label{corollary:inlier-error}
    Under the same assumptions as in Theorem \ref{thm:inlier_error}, 
    if we take the usual assumption that $\sigma_{\min,\text{in}} \geq \frac{c}{\sqrt r} \barzeta$        where $c$ is a constant
    as in \citet{TianGuFeng2023LearningFromSimilar,du2020few, niu2024collaborative}, 
    the upper bound can be written as:
\begin{align}\label{eq:beta-upperbound}
        \twonorm{\hbetak{t}-\bbetaks{t}}
    \lesssim 
    \bigg\{
    \sqrt{\frac{{k^\star}}{n}} + 
    \frac{\zetak{t}}{ \barzeta 
    \sqrt{1-\varepsilon}} 
    \sqrt{\frac{pr}{nT}}
    +
    \frac{\zetak{t}}{\barzeta}
    \sqrt{\frac{r}{n}}
    + \sqrt{\frac{\log T}{n}}
    &+
    \sqrt r h \zetak{t}
    \bigg[\frac{\sigmamax(\bm{D}^*_S)}{\sqrt{r}\sigmamin(\bm{D}^*_S)} \wedge 1\bigg]
    \bigg\}\\
   & \wedge 
    \sqrt{\frac{p+\log T}{n}},
    \quad \forall t \in S.
\end{align}
\end{corollary}

    As pointed out previously,
     under a mild assumption that the 
   contamination proportion $\varepsilon$ is upper bounded, 
    the term $\frac{1}{\sqrt{1-\varepsilon}}$
    is upper bounded by a constant factor which does not affect the rate.
    Besides, 
    the term $\frac{1}{\sqrt{1-\varepsilon}} $ combined with $\sqrt{\frac{pr}{nT}}$ can be written as $\sqrt \frac{pr}{n \cardS}$. If the number of inlier tasks $\cardS \gtrsim p$, we have $\sqrt \frac{pr}{n \cardS} \lesssim \sqrt \frac{r}{n}. $
    Note that in  Theorem \ref{thm:inlier_error} and Corollary \ref{corollary:inlier-error}, $k^\star$ is used as an upper bound for $\hat{k}$. In practice, $\hat{k}$ can adapt to the structure of inlier and outlier tasks, and may be smaller than $k^\star$.

    By contrast, the upper bound of \citet[Theorem 7]{TianGuFeng2023LearningFromSimilar} has an additive term 
    $\frac{\zetak{t}}{\barzeta}\max_{t \in S}\zetak{t}\cdot \sqrt{r\bar{\epsilon}} $ where $\bar{\epsilon}$ is an upper bound of the contamination proportion. 
    When the contamination proportion is large, this term dominates the upper bound, and becomes much larger than the single-task error rate.
    And it is assumed in \citet[Theorem 7]{TianGuFeng2023LearningFromSimilar} that
    $$\epsilon = \frac{|S^c|}{T} \leq cr^{-1}\cdot \Big(\frac{\barzeta}{\max_{t \in S}\zetak{t}}\Big)^2$$ where $c > 0$ is a small constant. This strictly limits the application scenarios where their method gains edge over the single-task regression.
    In contrast, our RAS approach can deal with very large contamination proportion, only requiring that 
    $\varepsilon<1$.

Now we consider when the error bound of RAS outperforms that of single-task regresion. 
Assume that $\zetak{t} \asymp 1, \zetabar \asymp 1,$ and $    \frac{\sigmamax(\bm{D}^*_S)}{\sigmamin(\bm{D}^*_S)} \asymp 1
$. The RAS brings benefit when the following conditions hold: 
$r_\textup{out} \ll p \;\textup{or} \; T\varepsilon \ll p, r \ll \cardS \wedge p, h\ll \sqrt{\frac{p+\log T}{n}}$. 
Since RAS tends to be beneficial when the latent representation heterogeneity $h$ is small, in implementations, we can choose $\tau =C \sqrt{\frac{p+\cardS}{nT}}$ and use cross-validation for choosing the constant $C$. 

Note that the first term in the upper bound of \eqref{eq:beta-upperbound} corresponds to the error bound of RAS without the biased regularization step (step 3 of Algorithm \ref{algo:ras}). Therefore, we may even omit the biased regularization step when the above conditions hold.

\section{Applications in Transfer Learning}\label{sec:TL}

Transfer learning (or meta-learning, learning to learn), aims to transfer the knowledge from some source tasks to the target task for efficient learning 
\citep{weiss2016survey,vanschoren2018meta,zhuang2020comprehensive}.
The learned shared representation in MTL is often used for transfer learning in the literature \citep{du2020few,TianGuFeng2023LearningFromSimilar, niu2024collaborative}.
In this section, we show how to leverage the learned representation $\Abarhat$ from RAS for efficient knowledge transfer to the target
task, and give its theoretical guarantees.

Assume that the target task also follows the linear model:
\begin{equation}\label{eq: linear model tl}
	\yk{0}_i = (\bxk{0}_i)^\top\bbetaks{0} + \epsilonk{0}_i 
\end{equation}
where $\bbetaks{0}=\bAks{0}\bthetaks{0}$, and  $\bAks{0} \in \mathcal{O}^{p \times r}$ with orthonormal columns is the latent representation corresponding to the target task.
The noise term $\epsilonk{0}_i$ is independent of the covariates $ \bxk{0}_i$. Denote $\zetak{0}:= \twonorm{\bbetaks{0}}.$
To describe the relationship between the target and source tasks, it is assumed that 
\begin{equation}
   \twonorm{
   \oA \oA^\top 
   - \bAks{0}(\bAks{0})^\top} \leq h_0,
\end{equation}
where $h_0$ measures the distance between the latent representation of the target task and the central representation of the  $T$ source tasks, and may differ from the $h$ defined in \eqref{eq:A similarity}.

Similar to the MTL setting, we make some standard assumptions below.
\begin{assumption}
\label{assump:TLcovariate}
For the target task, the feature vectors $\bm{x}_i^{(0)}$ are i.i.d. sub-Gaussian random vectors.
The covariance matrix is $\matr{\Sigma}^{(0)} = \E[\bm{x}_i^{(0)}(\bm{x}_i^{(0)})^\top]$, and we assume $0 < c_{\min} \le \lambda_{\min}(\matr{\Sigma}^{(0)}) \le \lambda_{\max}(\matr{\Sigma}^{(0)}) \le c_{\max} < \infty$, 
\end{assumption}

\begin{assumption}\label{assump:TLn}
	$n_0 \geq Cp$ with a  universal constant $C > 0$.
\end{assumption}

In the Algorithm \ref{algo: tl} below, we utilize the shared representation learned in our RAS algorithm, and 
apply a standard step for transferring the knowledge of the learned representation in the literature \citep{TianGuFeng2023LearningFromSimilar,du2020few,tripuraneni2021provable}.

\begin{algorithm}[!h]
\caption{RAS-Transfer}
\label{algo: tl}
\KwIn{
Target data $\{\bxk{0}_i, \yk{0}_i\}_{i=1}^{n_0}$, estimator $\hbarA$ from RAS,
penalty parameter $\gamma$}
\KwOut{Estimator $\hbetak{0}$}
\underline{Step 1:}  $\hthetak{0} = \argmin_{\btheta \in \mathbb{R}^{\hat{k}}}\big\{ \fk{0}(\hbarA\btheta)\big\}$ \\
\underline{Step 2:} $\hbetak{0} = \argmin_{\bbeta \in \mathbb{R}^p} \big\{\fk{0}(\bbeta) + \frac{\gamma}{\sqrt{n_0}}\twonorm{\bbeta - \hbarA\hthetak{0}}\big\}$
\end{algorithm}

Now, we show the theoretical guarantees for the coefficient estimation error of the target task.
\begin{proposition}
\label{thm:TL-error}
Let Assumptions \ref{assump:TLcovariate}-\ref{assump:TLn} hold.
Let threshold 
    $$
    \tau \asymp
    \sqrt{\frac{p+\cardS}{nT}}
    + h \barzeta \sqrt{1-\varepsilon}
    \bigg[\frac{\sigmamax(\bm{D}^*_S)}{\sqrt{r}\sigmamin(\bm{D}^*_S)} \wedge 1\bigg]
    $$
    in the RAS algorithm.
Let  
$\gamma =C' \sqrt{p+\log T}$ with a sufficiently large constant $C'$.
Then \wpr, the $\ell_2$ estimation error for the target task is bounded by:
\begin{align}
    \twonorm{\hbetak{0}-\bbetaks{0}}
    \lesssim 
    \bigg\{
    \sqrt{\frac{k^\star}{n_0}} + 
    h_0 \zetak{0}
    +
    \frac{\zetak{0} }{\sigma_{\min,\text{in}} }
    &\left(
    \frac{1}{\sqrt{1-\varepsilon} }
    \sqrt{\frac{p}{nT}}
    + \sqrt{\frac{1}{n}}
    \right)\\
    &+
    \frac{\zetak{0}}{\sigmaminin}
    \barzeta h 
    \bigg[\frac{\sigmamax(\bm{D}^*_S)}{\sqrt{r}\sigmamin(\bm{D}^*_S)} \wedge 1\bigg]    \bigg\}
    \wedge 
    \sqrt{\frac{p}{n_0}}.
\end{align}
\end{proposition}

\begin{corollary}
\label{corollary:TL-error}
    Under the same assumptions as in Proposition \ref{thm:TL-error},
    if we take the usual assumption that $\sigma_{\min,\text{in}} \geq \frac{c}{\sqrt r} \barzeta$        where $c$ is a constant
    as in \citet{TianGuFeng2023LearningFromSimilar,du2020few, niu2024collaborative}, 
    the upper bound can be written as:
\begin{align}
        \twonorm{\hbetak{0}-\bbetaks{0}}
    \lesssim 
    \bigg\{
    \sqrt{\frac{k^\star }{n_0}} + 
        h_0 \zetak{0}+
    \frac{\zetak{0}}{ \barzeta 
    \sqrt{1-\varepsilon}} 
    \sqrt{\frac{pr}{nT}}
    &+
    \frac{\zetak{0}}{\barzeta}
    \sqrt{\frac{r}{n}}\\
    &+
    \sqrt r h \zetak{0}
    \bigg[\frac{\sigmamax(\bm{D}^*_S)}{\sqrt{r}\sigmamin(\bm{D}^*_S)} \wedge 1\bigg]
    \bigg\}\\
   & \wedge 
    \sqrt{\frac{p}{n_0}}.
\end{align}
\end{corollary}
Similar to our discussion in the previous section, 
the upper bound is composed of two terms: the single-task rate, and the rate by the learned representation from our RAS algorithm. 
Assume that $\zetak{0} \asymp 1, \zetabar \asymp 1,$ and $    \frac{\sigmamax(\bm{D}^*_S)}{\sigmamin(\bm{D}^*_S)} \asymp 1
$.
The RAS-Transfer outperforms single-task regression, if
$r_\textup{out} \ll p\; \textup{or} \; T\varepsilon \ll p , h_0 \ll \frac{p}{n_0}, h \ll \frac{p}{n_0}, r \ll \frac{n\cardS}{n_0} \wedge p.$

\section{Numerical Experiments}
\label{sec:experiment}
This section presents numerical results to corroborate our theoretical findings.  The following methods are compared in the experiments.
\begin{itemize}
    \item Single-task regression (Single-task): we use the data from each task to fit the model.
    \item Pooled regression (Pooled) \citep{crammer2008learning}: we pool the data from all tasks and fit a regression model. 
    \item Robust Adaptive Spectral Method (RAS) proposed in the paper.
    \item Robust Adaptive Spectral Method without the biased regularization step (RAS w/o reg): our RAS algorithm, without the Step 3 (biased regularization).
    \item Robust Adaptive Spectral Method with winsorization (RAS w/ winsor): we use a winsorization step before conducting SVD.
    Specifically, let 
    $R = \texttt{quantile}(\{\twonorm{\widehat{\bbeta}^{(t)}}\}_{t=1}^T, 1-{\varepsilon})$. For tasks with $\twonorm{\hbetak{t}} > R,$ the $\hbetak{t}$ is rescaled to have $l_2$ norm $R$, and then the SVD is conducted.
    \item Spectral Method (SM): the Algorithm 2 in \citet{TianGuFeng2023LearningFromSimilar}.
    \item Spectral Method without the biased regularization step (SM w/o reg): we do not use the biased regularization step  (Step 4) in Algorithm 2 of \citet{TianGuFeng2023LearningFromSimilar},
    which demonstrates the performance of the learned shared representation of SM.
\end{itemize}
Note that SM and SM w/o reg requires the knowledge of the true $r$, which is assumed to be known in the experiments. 
For the winsorization step used in SM, SM w/o reg, and RAS w/ winsor, we assume the true contamination proportion $\varepsilon$ is known.
The threshold in RAS, RAS w/o reg, RAS w/ winsor is set as $\tau = 2.5 \sqrt{\frac{p+T}{nT}}$. 
For the biased regularization step used in RAS, RAS w/ winsor, SM,  we follow the regularization parameter $\gamma = 0.5 \sqrt{p+\log T}$ as  in \citet{TianGuFeng2023LearningFromSimilar}, and use the automatic differentiation implementation in PyTorch \citep{paszke2019pytorch}, with Adam optimizer \citep{kingma2015adam}, and the learning rate is set to 0.01.

In the following sections, we present the performance of different methods, with low-rank or general-rank outliers. We also explore the impact of winsorization, especially when the outlier magnitude is small, and the impact of increasing number of tasks in MTL.

\subsection{Low-rank Outliers}
\label{sec:low-rank-outliers}
In this section, we explore the performance of different methods, with approximately low-rank outliers.
The data generating process is as follows.
Among the $T$ tasks, we randomly sample $T \varepsilon$ (with rounding) as outliers. 
For each task, the covariates $\bxk{t}_i$ are i.i.d. from $N(\bm{0}_p, \bm{I}_p)$, with the noise term $\epsilonk{t}_i$ i.i.d. from $N(0, 1)$.
We set a matrix $Q$ as the orthonormal matrix in the QR decomposition of a $p \times p$ random matrix with i.i.d. standard normal entries.
Let $\oA$ be the first $r$ columns of the matrix $Q$, and $\Aout$ be the next $r_{\textup{out}}$ columns.
The true coefficient $\bbetaks{t}$ of inlier tasks are generated as $\bbetaks{t} = \oA \bthetaks{t}$ where the entries of $\bthetaks{t}$ are i.i.d. from $\textup{Unif}[-a,a]$, with $a=2$.
For outlier tasks, we generate 
$\bbetaks{t} = \Aout \bthetaks{t}$ where the entries of $\bthetaks{t}$ are i.i.d. from $\textup{Unif}[-b,b]$, with $b=20$, to represent adversarially contaminated tasks.
The response variable is then generated according to the linear model \eqref{eq:linear model}.
We vary the contamination proportion $\varepsilon$ among $[0,0.05,0.1, \ldots,0.75,0.8]$, with each value of $\varepsilon$ replicated for 100 times and the average being taken.
We consider the following parameter settings.
\begin{itemize}
    \item $n=150, p=80, T=100, r=10, r_{\textup{out}}=10.$
    \item $n=100, p=50, T=100, r=5, r_{\textup{out}}=5.$
\end{itemize}
The results are shown in Figures \ref{fig:lowrank-n150} and \ref{fig:lowrank-n100}.
When the contamination proportion is slightly large (e.g., when $\varepsilon \geq 0.1$), the inlier estimation error of SM w/o reg goes up quickly, being much higher than the single-task estimation error, which shows that the shared representation learned by SM is off. While the SM w/o reg has much higher estimation error than the Single-task, the SM has similar  performance with Single-task, because the biased regularization  is used, which prevents negative transfer from the poorly learned shared representation.

In the data generating process of this section, the $r_{\textup{out}}$ 
is set to be small, which falls into the approximately low-rank contamination setting. 
According to Lemma \ref{lemma:low rank outlier} and Theorem \ref{thm:inlier_error}, the estimated dimension $\hat{k}$ is small, and the RAS should have small estimation error. 
The numerical results here corroborate the theoretical statements.
Note that RAS consistently outperforms single-task regression, in the large range of contamination proportion from $0$ to $0.8$, showcasing the effectiveness of our proposed method. 

Also, with the increase of the contamination proportion, there is mild increase of the estimation error of RAS  due to the $\frac{1}{\sqrt{1-\varepsilon}}$ factor in the upper bound of estimation error, which demonstrates our previous argument that the multiplicative factor term has minor influence on the performance.

Besides, it is worth noting that RAS w/o reg shows very similar performance as RAS, which demonstrates that the learned representation from our RAS algorithm well captures the true inlier subspace $\oA,$ and thus outperforms the single-task regression, even without the biased regularization step.

Another observation is that RAS w/ winsor exhibits similar performance as RAS and RAS w/o reg. This is because the magnitude of the contaminated tasks is set to be much larger than the inlier tasks, and thus the winsorization step does not affect the observed inlier signal.
The pooled regression consistently have the highest estimation error, showing that simple data pooling does not work well in the existence of contamination.

The right panel of Figures \ref{fig:lowrank-n150} and \ref{fig:lowrank-n100} shows the estimated dimension $\hat{k}$ of RAS and RAS w/ winsor. Since the outlier magnitude is quite large, with slight increase of the contamination proportion, RAS and its with-winsorization variant detect the combined signals and use them in the learning of the shared representation, which leads to the superior performance demonstrated in the estimation error in the left panel. 
Besides, when the contamination gets very large, e.g., $\tau=0.8,$ the estimated dimension $\hat{k}$ slightly drops. This is because when the contamination proportion gets very large, the number of inlier tasks gets very few, e.g., there are only $20$ inlier tasks in our setting, when $\varepsilon=0.8.$ And this is a little insufficient to learn the true $r$-dim inlier subspace. Some signals that are not well represented in the observed inlier tasks shrink below the threshold, resulting in the minor decrease in the estimated dimension.

\begin{figure}[htbp]
  \centering
  \includegraphics[width=\linewidth]{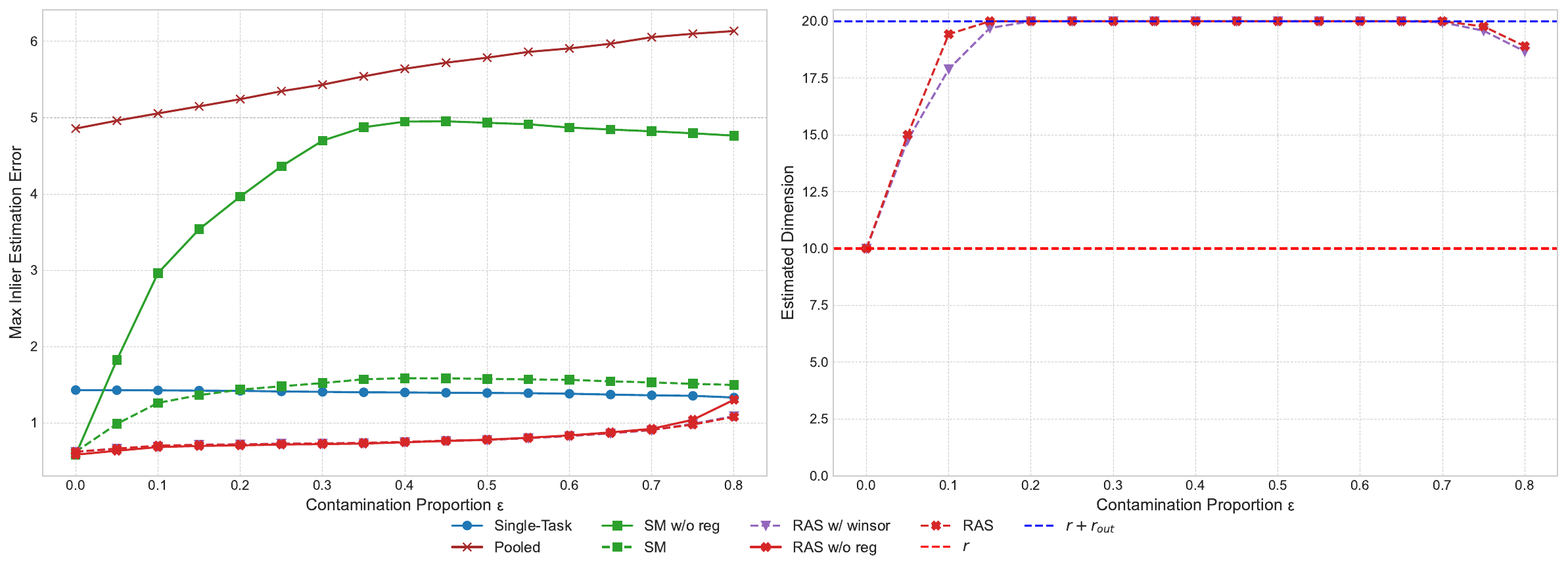}
  \caption{
  $n=150, p=80, T=100, r=10, r_{\textup{out}}=10.$
  In the left panel, the vertical axis is the maximum inlier estimation error $\max_{t \in S}\twonorm{\hbetak{t}-\bbetaks{t}}$.
   The horizontal axis is the contamination proportion.
   In the right panel, the vertical axis is the estimated dimension $\hat{k}$, and the horizontal axis is the contamination proportion.
   Each point is computed as the average of 100 replications.}
  \label{fig:lowrank-n150} %
\end{figure}
\begin{figure}[htbp]
  \centering
  \includegraphics[width=\linewidth]{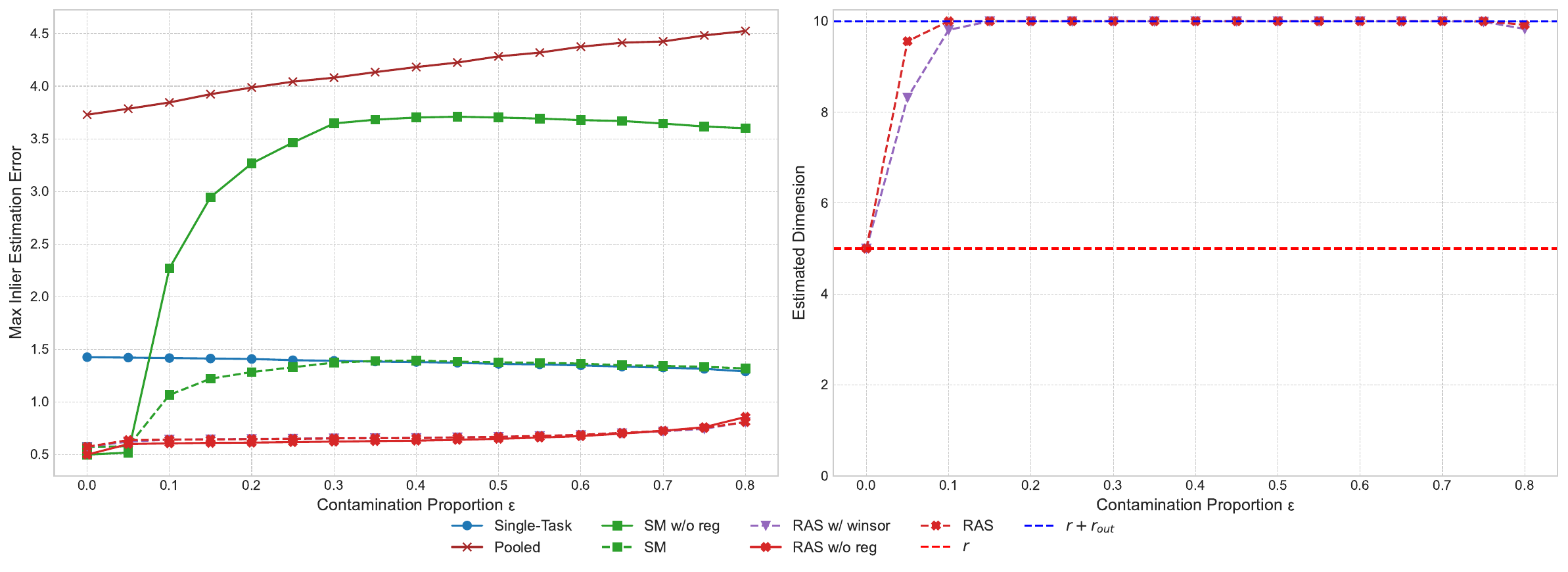}
  \caption{
  $n=100, p=50, T=100, r=5, r_{\textup{out}}=5.$
  In the left panel, the vertical axis is the maximum inlier estimation error $\max_{t \in S}\twonorm{\hbetak{t}-\bbetaks{t}}$.
   The horizontal axis is the contamination proportion.
   In the right panel, the vertical axis is the estimated dimension $\hat{k}$, and the horizontal axis is the contamination proportion.
   Each point is computed as the average of 100 replications.}
  \label{fig:lowrank-n100} %
\end{figure}

\subsection{Impact of Winsorization}

As we observed in the previous section, it seems that the with-winsorization version of RAS performs similarly with RAS. However, we argue that this is not always the case, especially when the outlier magnitude is small, and the winsorization may affect the observed inlier signals. 
In this section, we explore the impact of winsorization on RAS.

We adopt a similar data generating process, with outliers of smaller magnitude, to exhibit the impact of winsorization clearly.
Specifically, we set $b=0.2$ in the data generating process of Section \ref{sec:low-rank-outliers}, and keep other settings the same.
We consider the following parameter settings.
\begin{itemize}
    \item $n=150, p=80, T=100, r=10, r_{\textup{out}}=10.$
    \item $n=100, p=50, T=100, r=5, r_{\textup{out}}=5.$
\end{itemize}

\begin{figure}[htbp]
  \centering
  \includegraphics[width=\linewidth]{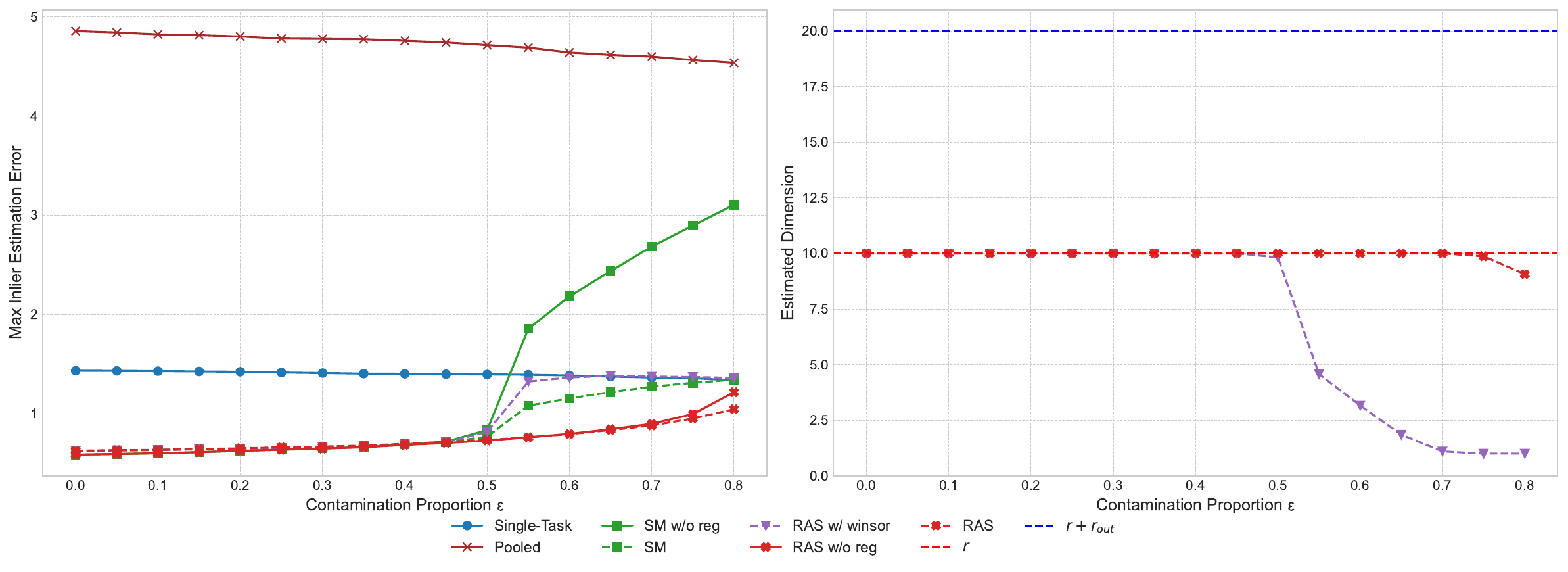}
  \caption{
  $n=150, p=80, T=100, r=10, r_{\textup{out}}=10.$
  In the left panel, the vertical axis is the maximum inlier estimation error $\max_{t \in S}\twonorm{\hbetak{t}-\bbetaks{t}}$.
   The horizontal axis is the contamination proportion.
   In the right panel, the vertical axis is the estimated dimension $\hat{k}$, and the horizontal axis is the contamination proportion.
   Each point is computed as the average of 100 replications.}
  \label{fig:winsor-n150} %
\end{figure}

\begin{figure}[htbp]
  \centering
  \includegraphics[width=\linewidth]{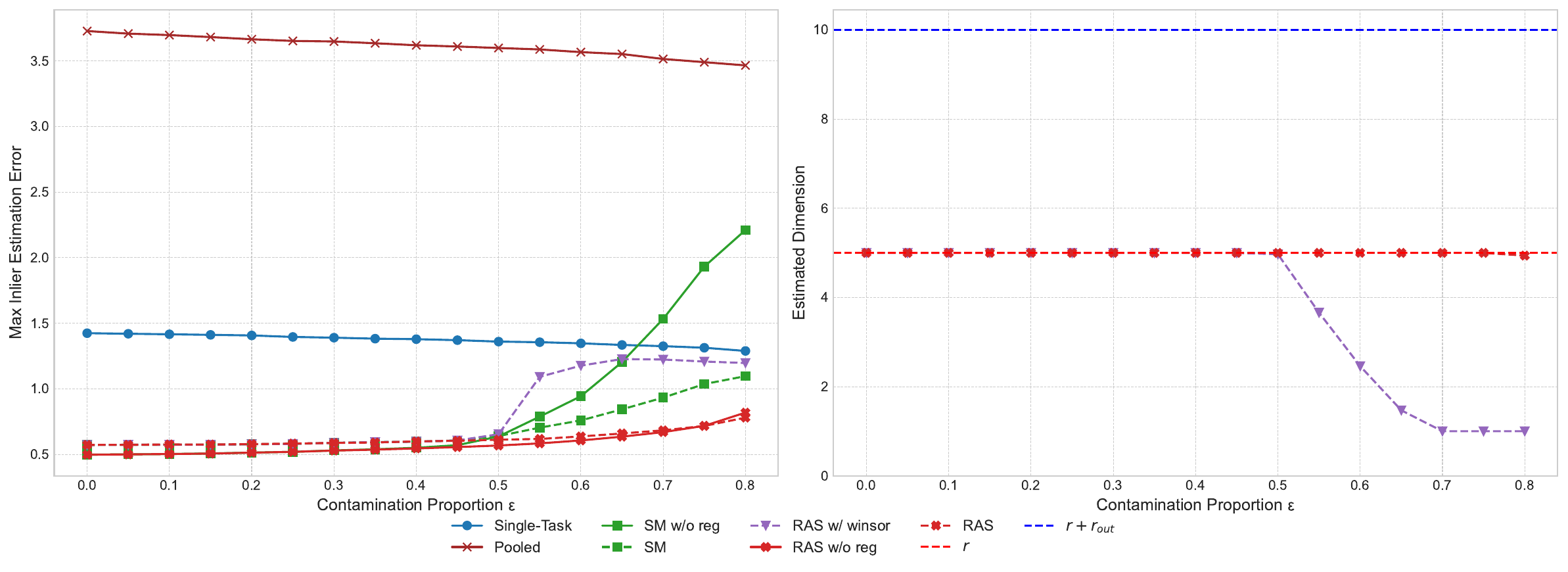}
  \caption{
  $n=100, p=50, T=100, r=5, r_{\textup{out}}=5.$
  In the left panel, the vertical axis is the maximum inlier estimation error $\max_{t \in S}\twonorm{\hbetak{t}-\bbetaks{t}}$.
   The horizontal axis is the contamination proportion.
   In the right panel, the vertical axis is the estimated dimension $\hat{k}$, and the horizontal axis is the contamination proportion.
   Each point is computed as the average of 100 replications.}
  \label{fig:winsor-n100} %
\end{figure}

The results are shown in Figures \ref{fig:winsor-n150} and \ref{fig:winsor-n100}. Note that the magnitude of outliers is set to be very small, so when the contamination proportion is moderate, the outliers do not affect the performance of these methods.
When $\varepsilon >0.5,$ all the inlier task coefficients are winsorized, causing inlier signal losses.  Therefore,  methods that use winsorization (RAS w/ winsor, SM, SM w/o reg) quickly have higher inlier estimation errors. Because RAS w/ winsor and SM use the biased regularization step to prevent negative transfer, they show similar or a little lower estimation error than single-task regression, but perform worse than RAS and RAS w/o reg.
The right panel of the Figures \ref{fig:winsor-n150} and \ref{fig:winsor-n100} shows the estimated dimension $\hat{k}$. When $\varepsilon > 0.5,$ the $\hat{k}$ declines quicly, since some of the winsorized inlier signals drop below the threshold. 
Note that we use oracle SM and SM w/o reg, with the true $r$ in the experiments. But similar effects of winsorization on the estimated dimension, as being observed for RAS, also apply to SM and SM w/o reg without oracle knowledge of the true $r$.

On the other hand, the RAS and RAS w/o reg are quite robust, regardless of the outlier magnitude, since they do not use winsorization and avoid the risk of inlier signal loss. Also, note that RAS and RAS w/o reg is quite adaptive to the outlier signal magnitude: when the outlier magnitude is small, the estimated dimension $\hat{k}$ recovers the true $r$, as opposed to the combined dimension $r+r_{\textup{out}}$.

Also, note that the exact contamination proportion $\varepsilon$ is assumed to be known in the experiments, and the winsorization uses the true value.
In practice, however, people only get a conservative upper bound for the contamination proportion $\varepsilon$ oftentimes,  which even amplifies the impact of winsorization.

\subsection{General-rank Outliers}
In this section, we explore the performance of different methods when the outliers does not enjoy the approximate low-rank structure.
We use the same data generating distribution as in Section \ref{sec:low-rank-outliers}. We consider the following parameter setting.
\begin{itemize}
    \item $n=150, p=80, T=100, r=5, r_{\textup{out}}=65.$
\end{itemize}

\begin{figure}[htbp]
  \centering
  \includegraphics[width=\linewidth]{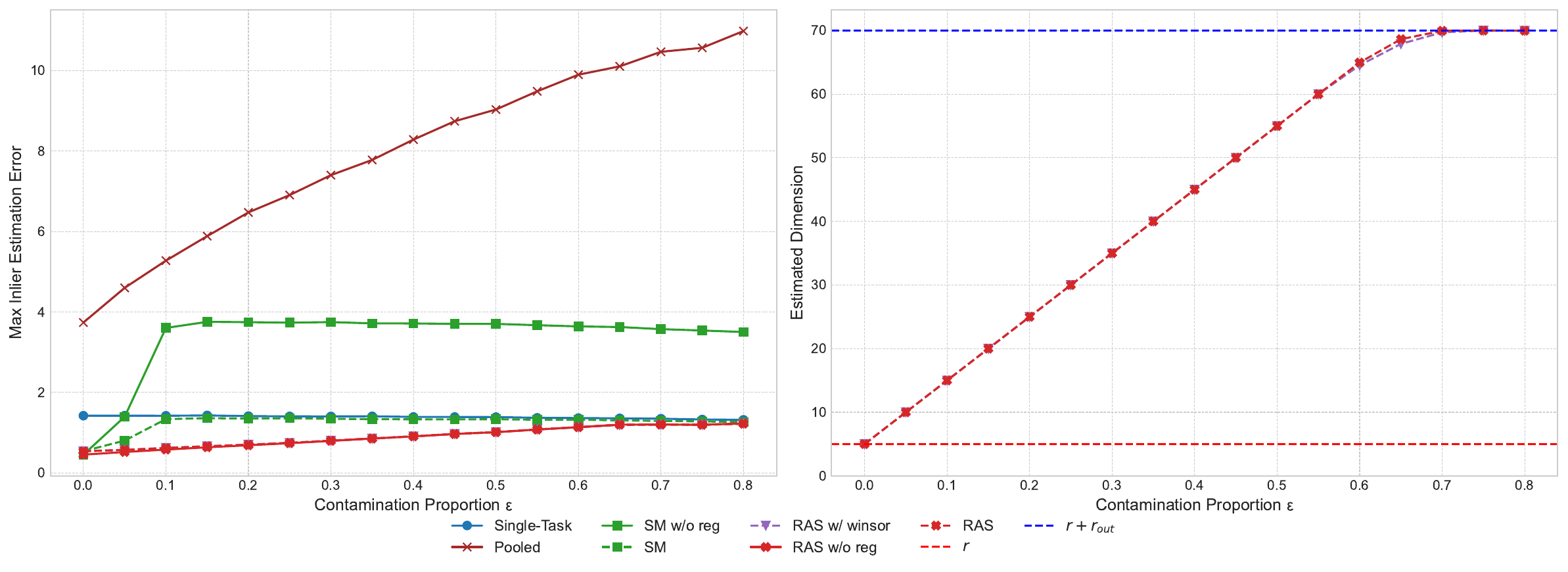}
  \caption{
    $n=150, p=80, T=100, r=5, r_{\textup{out}}=65.$
  In the left panel, the vertical axis is the maximum inlier estimation error $\max_{t \in S}\twonorm{\hbetak{t}-\bbetaks{t}}$.
   The horizontal axis is the contamination proportion.
   In the right panel, the vertical axis is the estimated dimension $\hat{k}$, and the horizontal axis is the contamination proportion.
   Each point is computed as the average of 100 replications.}
  \label{fig:general-rank} %
\end{figure}

The results are shown in Figure \ref{fig:general-rank}.
With slightly large contamination proportion, SM w/o reg performs worse than single-task regression, since the learned representation is corrupted by the contamination. SM, with the biased regularization step, achieve similar performance as single-task regression.
This shows that with general-rank outliers, using SM does not bring benefit with the existence of even small amount of contaminated tasks.

In contrast, RAS, RAS w/o reg consistently outperform the single-task regression, with the performance gain adaptive to the contamination proportion.
As $\varepsilon$ increases, more outlier signals surpass the threshold $\tau$, gradually increasing the estimated dimension $\hat{k}$.
And as our Theorem \ref{thm:inlier_error} demonstrates, the inlier estimation error increases.
Note that even with modest to large level of contamination proportion, RAS still shows meaningful gains over the single-task regression and SM.

Since the outlier magnitude is much larger than that of the inliers, the winsorization step does not harm the inliers, and thus RAS w/ winsor shows similar performance as RAS.
Pooled regression still shows the worst performance as expected.

\subsection{Impact of Different Number of Tasks}
In this section, we present results with varing number of tasks. We follow the same data generating process as in Section \ref{sec:low-rank-outliers}. 
We consider the following parameter setting.
\begin{itemize}
    \item $n=100, p=50, \varepsilon=0.5, r=5, r_{\textup{out}}=5.$
\end{itemize}
We vary $T$ from $25$ to $190$ with step size $15$.

\begin{figure}[htbp]
  \centering
  \includegraphics[width=\linewidth]{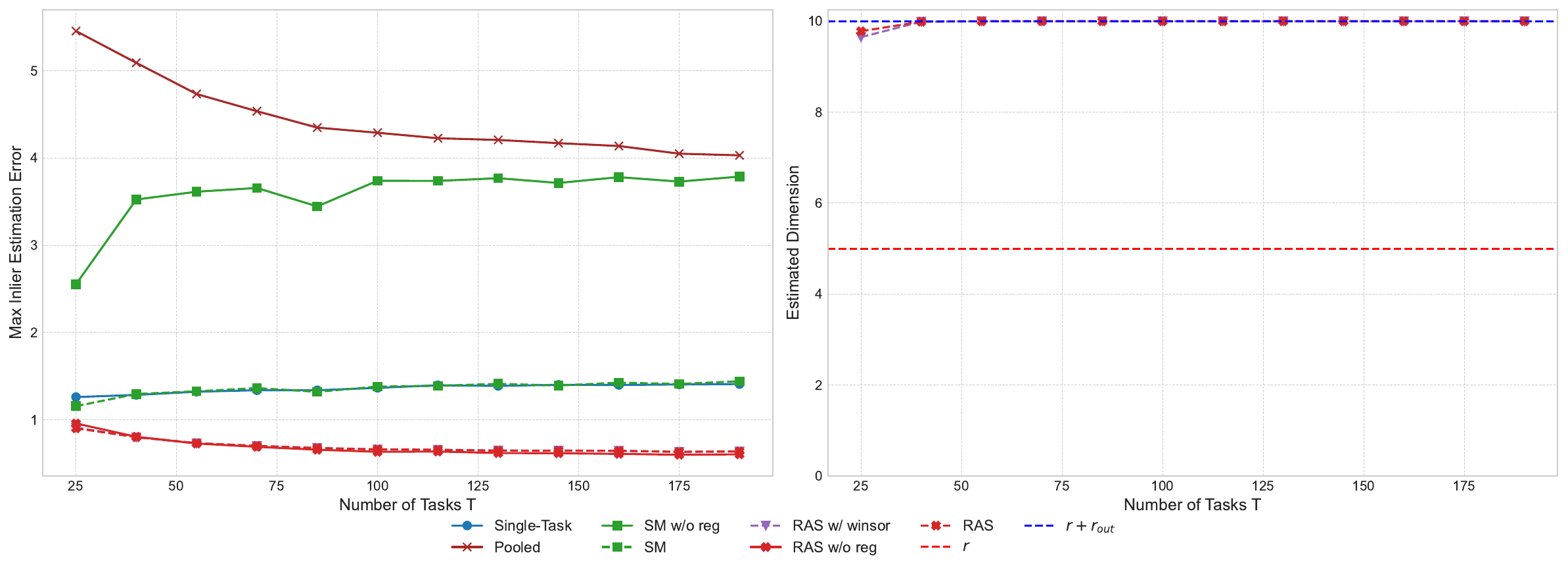}
   \caption{
  $n=100, p=50, \varepsilon=0.5, r=5, r_{\textup{out}}=5.$
  In the left panel, the vertical axis is the maximum inlier estimation error $\max_{t \in S}\twonorm{\hbetak{t}-\bbetaks{t}}$.
   The horizontal axis is the number of tasks .
   In the right panel, the vertical axis is the estimated dimension $\hat{k}$, and the horizontal axis is the number of tasks.
   Each point is computed as the average of 100 replications.}
  \label{fig:varyT} %
\end{figure}

The results are shown in Figure \ref{fig:varyT}. Note that with increasing $T$, the inlier estimation error of RAS and RAS w/o reg first decrease and then stabilize. This matches the upper bound in Theorem \ref{thm:inlier_error} and Corollary \ref{corollary:inlier-error}. 
With $h=0$, and ignoring the log factors, let us consider the following term in the upper bound:
$\sqrt{\frac{\hat{k}}{n}} + 
    \frac{\zetak{t}}{ \barzeta 
    } \left( 
    \frac{1}{\sqrt{1-\varepsilon}}
    \sqrt{\frac{pr}{nT}}
    +
    \sqrt{\frac{r}{n}}
    \right)$.
With $T$ increasing, the term $\sqrt{\frac{pr}{nT}}$ decreases. And when $T$ gets large enough, the term regarding $\sqrt{\frac{\hat{k}}{n}}$ dominates the upper bound, so the inlier estimation error stablizes. The RAS w/ winsor performs similarly with RAS and RAS w/o reg, since the magnitude of outliers is large relative to the inliers, and winsorization mostly operate on the outliers.

On the other hand, the maximum inlier estimation error of SM and SM w/o reg does not decrease, and even increases slightly, with larger $T$. This is because the contaminated tasks ruin the learned representation of SM and SM w/o reg. And with the help of the biased regularization step using single-task data, the performance of SM equals that of single-task regression. With increasing $T$, the maximum inlier estimation error of single-task regression has slight increase, and so does SM.

The right panel of Figure \ref{fig:varyT} shows the relationship between the estimated dimension $\hat{k}$ and $T$.
When $T=25$, with $\varepsilon=0.5$, the number of inliers and outliers are both small, therefore among the $100$ repetitions, RAS occasionally misses a small proportion of the combined signal dimension, leading to an average estimated dimension slightly below $r+r_{\textup{out}}$. However, for all the other larger $T$ values, RAS can recover all the combined signal dimensions in all the repetitions.

\section{Discussion}\label{sec:discussion}

In this paper, we address a critical and practical challenge in multi-task learning: the presence of an unknown and potentially large fraction of contaminated tasks. We introduce the Robust and Adaptive Spectral method (RAS), a computationally efficient algorithm that learns a shared representation from a collection of tasks without prior knowledge of the contamination level or the true underlying rank. Our theoretical analysis provides non-asymptotic error bounds that guarantee robustness, adaptivity to the structure of both inliers and outliers, and protection against negative transfer. Extensive experiments corroborate our theory, demonstrating that RAS maintains strong performance even when up to 80\% of the tasks are contaminated, significantly outperforming existing methods that rely on oracle knowledge or fragile heuristics.

In conclusion, RAS provides a robust, adaptive, and theoretically-grounded framework for representation learning in contaminated multi-task learning and transfer learning settings.
By moving beyond the clean-data paradigm, this work helps bridge the gap between MTL/TL theory and its practical application in complex real-world environments, opening up new possibilities for reliable knowledge sharing across diverse and potentially corrupted sources of data.

\bibliographystyle{ims}
\bibliography{reference}

\newpage 
\appendix
\numberwithin{equation}{section}
\section{Proof of Results in Sections \ref{sec:theory} and \ref{sec:TL}}

\subsection{Proof of Proposition \ref{thm:rank_general}
}
\begin{proof}
Recall that
$$\bdeltaks{t} = \oA^\perp(\oA^\perp)^\top
\bbetaks{t}=\oA^\perp(\oA^\perp)^\top
\bAks{t}\bthetaks{t},
$$
thus we have 
\begin{equation}\label{eq:delta-t-star-bd}
\twonorm{\bdeltaks{t}} \leq \twonorm{ (\oA)^\perp  ((\oA)^\perp)^\top 
\bAks{t}\bthetaks{t}} \leq \twonorm{\bAks{t}(\bAks{t})^\top - \oA(\oA)^\top}\zetak{t} \leq h\zetak{t}.
\end{equation}
Note that 
$$
\frac{1}{\sqrt{T}}\twonorm{\bm{D}^*_S} \leq \frac{1}{\sqrt{T}}\fnorm{\bm{D}^*_S}
\bigg[
\frac{\sigmamax(\bm{D}^*_S)}{\sqrt{r}\sigmamin(\bm{D}^*_S)} 
\wedge 1 \bigg],
$$
and 
$ \frac{1}{\sqrt T}
\fnorm{\bm{D}_S^*}
\leq \sqrt{1-\varepsilon} h \barzeta.$
By Lemma \ref{lemma:noise_opnorm},
 \wpp,
\begin{equation}
   \frac{1}{\sqrt T}
    \opnorm{\bN} \lesssim 
    \sqrt{\frac{p+\cardS}{nT}}
    +
    h \barzeta \sqrt{1-\varepsilon}
\bigg[\frac{\sigmamax(\bm{D}^*_S)}{\sqrt{r}\sigmamin(\bm{D}^*_S)} \wedge 1\bigg]    
\end{equation}

Therefore, by our choice of 
$$
    \tau 
    = C'
    \left(
    \sqrt{\frac{p+\cardS}{nT}}
    + h \barzeta \sqrt{1-\varepsilon}
    \bigg[\frac{\sigmamax(\bm{D}^*_S)}{\sqrt{r}\sigmamin(\bm{D}^*_S)} \wedge 1\bigg]
    \right)
    $$
 with a sufficiently large constant 
 $C'$,
we have, \wpp,
\begin{equation}
    \opnorm{\matr{N}/\sqrt{T}} 
    < 0.25 \tau. 
\end{equation}

Let $\hat{\lambda}_k = \sigma_k(\widehat{\matr{B}}_{\text{st}}/\sqrt{T})$ be the observed singular values and $\widetilde{\lambda}_k = \sigma_k( \widetilde{\matr{B}}  /\sqrt{T})$ be the true singular values. By Weyl's inequality, for any $k$, we have:
\begin{equation}
    |\hat{\lambda}_k - \widetilde{\lambda}_k| \le \opnorm{\matr{N}/\sqrt{T}} < 0.25\tau.
\end{equation}

The rank $\hat{k}$ is estimated as the number of singular values $\hat{\lambda}_k$ greater than $\tau$. We check the singular values around the index $k = r_{\text{eff}}$.

{At $k = r_{\text{eff}}$}:
By the definition of $r_{\text{eff}}$, we have $\widetilde{\lambda}_{r_{\text{eff}}} > 1.25\tau$. Using Weyl's inequality:
\begin{equation}
\hat{\lambda}_{r_{\text{eff}}} \ge \widetilde{\lambda}_{r_{\text{eff}}} - \opnorm{\matr{N}/\sqrt{T}} > 1.25\tau - 0.25\tau  =  \tau.
\end{equation}
This implies that at least $r_{\text{eff}}$ singular values of the observed matrix will be above the threshold $\tau$.

{At $k = r_{\text{eff}}+1$}:
By the spectral gap assumption,
we have $\widetilde{\lambda}_{r_{\text{eff}}+1} < 0.75\tau$. Using Weyl's inequality:
\begin{equation}
\hat{\lambda}_{r_{\text{eff}}+1} \le \widetilde{\lambda}_{r_{\text{eff}}+1} + \opnorm{\matr{N}/\sqrt{T}} < 0.75\tau + 0.25\tau = \tau.
\end{equation}
This concludes the proof.
\end{proof}

\subsection{Proof of Lemma \ref{lemma:low rank outlier}}
\begin{proof}
    Let $\widebar{\bB}=[\widebar{\bB}_S, \widebar{\bB}_{S^c}]$. By the dimension formula for the sum of vector spaces, 
    $\textup{rank}(\widebar{\bB})=r + r_{\textup{out}} - r_\cap$.
    Note that $\widetilde{\bB}=
    \widebar{\bB} + [\bm{0} \,, \widetilde{\bN}].$ Thus, by Weyl's inequality, we have, 
    \begin{equation}    \abs{\sigma_j(\widetilde{\bB}) 
        - \sigma_j(\widebar{\bB})} \leq \opnorm{\widetilde{\bN}} \leq \frac{3}{4} \tau. 
    \end{equation}
    Since $\card{ \{ r \mid \sigma_r(\widebar{\bB}) > 
    0
    \}} \leq 
    \text{rank}(\widebar{\bB})$, we have, $\card{ \{ r \mid \sigma_r(\widetilde{\bB}) > \frac{3}{4} \tau \}} \leq 
    r+r_{\textup{out}}-r_\cap.
    $
\end{proof}

\subsection{Proof of Proposition \ref{prop:subspace_error_bd_general}}
\begin{proof}

Let $\Abarhat^\perp$ be an orthonormal basis for the orthogonal complement of $\mathrm{span}(\Abarhat)$. We have,
\begin{equation}
    \opnorm{(\matr{I} - {\matr{P}}_{\Abarhat}) \oA} = \opnorm{(\Abarhat^\perp)^\top \oA}.
\end{equation}

By definition of the SVD, the columns of $\Abarhat$ are the top $\hat{k}$ left singular vectors of $\widehat{\matr{B}}_{\text{st}}$. Let $\widehat{\matr{B}}_{\text{st}} = \sum_{j=1}^{\min(p,T)} \sigma_j u_j v_j^\top$. Then $\Abarhat = [u_1, \dots, u_{\hat{k}}]$ and we can choose $\Abarhat^\perp = [u_{\hat{k}+1}, \dots, u_p]$.
 We have:
\begin{equation}
(\Abarhat^\perp)^\top \widehat{\matr{B}}_{\text{st}} = \begin{pmatrix} u_{\hat{k}+1}^\top \\ \vdots \\ u_p^\top \end{pmatrix} \sum_{j=1}^{\min(p,T)} \sigma_j u_j v_j^\top = \begin{pmatrix} \sigma_{\hat{k}+1} v_{\hat{k}+1}^\top \\ \vdots \\ \sigma_p v_p^\top \end{pmatrix}.
\end{equation}
The operator norm of this matrix is $\opnorm{(\Abarhat^\perp)^\top \widehat{\matr{B}}_{\text{st}}} = \sigma_{\hat{k}+1}(\widehat{\matr{B}}_{\text{st}})$. From our rank estimation procedure,  $\sigma_{\hat{k}+1}(\widehat{\matr{B}}_{\text{st}}) < \sqrt{T}\tau$.

Recall that
$$\bdeltaks{t} = \oA^\perp(\oA^\perp)^\top
\bbetaks{t}=\oA^\perp(\oA^\perp)^\top
\bAks{t}\bthetaks{t},
$$
thus we have 
\begin{equation}
\twonorm{\bdeltaks{t}} \leq \twonorm{ (\oA)^\perp  ((\oA)^\perp)^\top 
\bAks{t}\bthetaks{t}} \leq \twonorm{\bAks{t}(\bAks{t})^\top - \oA(\oA)^\top}\zetak{t} \leq h\zetak{t}.
\end{equation}
Note that 
$$
\frac{1}{\sqrt{T}}\twonorm{\bm{D}^*_S} \leq \frac{1}{\sqrt{T}}\fnorm{\bm{D}^*_S}
\bigg[
\frac{\sigmamax(\bm{D}^*_S)}{\sqrt{r}\sigmamin(\bm{D}^*_S)} 
\wedge 1 \bigg],
$$
and 
$ \frac{1}{\sqrt T}
\fnorm{\bm{D}_S^*}
\leq \sqrt{1-\varepsilon} h \barzeta, $
we have, \wpp,
\begin{equation}
   \frac{1}{\sqrt T}
    \twonorm{\bN} \lesssim 
    \sqrt{\frac{p+\cardS}{nT}}
    +
    h \barzeta \sqrt{1-\varepsilon}
\bigg[\frac{\sigmamax(\bm{D}^*_S)}{\sqrt{r}\sigmamin(\bm{D}^*_S)} \wedge 1\bigg]    
\end{equation}

Now, we use the identity $\widehat{\matr{B}}_{\text{st}} = \widetilde{\bm{B}} + \matr{N}$:
\begin{equation}
(\Abarhat^\perp)^\top \widetilde{\bm{B}} = (\Abarhat^\perp)^\top \widehat{\matr{B}}_{\text{st}} - (\Abarhat^\perp)^\top \matr{N}.
\end{equation}
Let us consider only the columns of this matrix equation corresponding to the inlier tasks, $t \in S$. Let $\matr{\Pi}_S \in \R^{T \times |S|}$ be a selection matrix for these columns.
\begin{equation} \label{eq:key_identity}
(\Abarhat^\perp)^\top \widetilde{\bm{B}} \matr{\Pi}_S = (\Abarhat^\perp)^\top \widehat{\matr{B}}_{\text{st}} \matr{\Pi}_S - (\Abarhat^\perp)^\top \matr{N} \matr{\Pi}_S.
\end{equation}
The left hand side is exactly $(\Abarhat^\perp)^\top \widetilde{\bm{B}}_S = ((\Abarhat^\perp)^\top \oA) \matr{\Theta}_S^*$. Taking operator norms of both sides of \eqref{eq:key_identity}:
\begin{equation}
\opnorm{((\Abarhat^\perp)^\top \oA) \matr{\Theta}_S^*} \le \opnorm{(\Abarhat^\perp)^\top \widehat{\matr{B}}_{\text{st}} \matr{\Pi}_S} + \opnorm{(\Abarhat^\perp)^\top \matr{N} \matr{\Pi}_S}.
\end{equation}
We bound the terms:
\begin{itemize}
    \item $\opnorm{((\Abarhat^\perp)^\top \oA) \matr{\Theta}_S^*} \ge \opnorm{(\Abarhat^\perp)^\top \oA} \cdot \sigma_r(\matr{\Theta}_S^*)$.
    \item $\opnorm{(\Abarhat^\perp)^\top \widehat{\matr{B}}_{\text{st}} \matr{\Pi}_S} \le \opnorm{(\Abarhat^\perp)^\top \widehat{\matr{B}}_{\text{st}}} \le \sigma_{\hat{k}+1}(\widehat{\matr{B}}_{\text{st}}) < \sqrt{T}\tau$.
    \item $\opnorm{(\Abarhat^\perp)^\top \matr{N} \matr{\Pi}_S} \le \opnorm{\matr{N}} \lesssim \sqrt{T}\tau$.
\end{itemize}
Combining these inequalities, we have:
\begin{equation}
\opnorm{(\Abarhat^\perp)^\top \oA} \cdot \sigma_r(\matr{\Theta}_S^*) \lesssim \sqrt{T}\tau.
\end{equation}
Using Assumption \ref{assump:diversity}, $\sigma_r(\matr{\Theta}_S^*) \ge \sqrt{\card{S}}\sigma_{\min,\text{in}}$. This gives the  subspace error bound:
\begin{equation}
\opnorm{(\matr{I} - {\matr{P}}_{\Abarhat}) \oA} = \opnorm{(\Abarhat^\perp)^\top \oA} \lesssim \frac{\sqrt{T}\tau}{\sqrt{\card{S}}\sigma_{\min,\text{in}}} = \frac{\tau}{\sigma_{\min,\text{in}} \sqrt{1-\varepsilon} }.
\end{equation}
Substituting the definition of $\tau$:
\begin{equation}
\opnorm{(\matr{I} - {\matr{P}}_{\Abarhat}) \oA} \lesssim \frac{1}{\sigma_{\min,\text{in}} \sqrt{1-\varepsilon}} 
\left(
\sqrt{\frac{p+\cardS}{nT}}
+h \barzeta \sqrt{1-\varepsilon}
\bigg[\frac{\sigmamax(\bm{D}^*_S)}{\sqrt{r}\sigmamin(\bm{D}^*_S)} \wedge 1\bigg]
\right).
\end{equation}
\end{proof}

\subsection{Proof of Theorem \ref{thm:inlier_error}}
\begin{proof}
Let ${\matr{P}}_{\Abarhat} = \Abarhat\Abarhat^\top$ be the projector onto the estimated subspace. Define the projection of the true coefficient vector onto this subspace, $\bbeta^{(t)}_{\text{proj}} = {\matr{P}}_{\Abarhat} \bbeta^{(t)*}$.
Recall that 
$\hthetak{t} = \argmin_{\btheta \in \mathbb{R}^{\hat{k}}}\fk{t}
(\Abarhat \btheta).$
By the triangle inequality, 
\begin{equation}
\norm{\Abarhatthetahatt - \bbeta^{(t)*} }_2 \le 
\norm{\Abarhatthetahatt - \bbeta^{(t)}_{\text{proj}} }_2 
+ 
\norm{\bbeta^{(t)}_{\text{proj}} - \bbeta^{(t)*} }_2.
\end{equation}
The approximation error term
\begin{equation}
\norm{\bbeta^{(t)}_{\text{proj}} - \bbeta^{(t)*} }_2 = \norm{({\matr{P}}_{\Abarhat} - \matr{I}) \bbeta^{(t)*} }_2
\leq \twonorm{\bbetaks{t}-\bP_{\oA} \bbetaks{t}}+ 
\twonorm{\bP_{\oA} \bbetaks{t} - \bP_{\hbarA} \bbetaks{t}}.
\end{equation}
By \eqref{eq:delta-t-star-bd}, $\twonorm{\bbetaks{t}-\bP_{\oA} \bbetaks{t}} \leq h \zetak{t}.$
Note that, by Proposition \ref{prop:subspace_error_bd_general},
\begin{equation}
    \twonorm{\bP_{\oA} \bbetaks{t} - \bP_{\hbarA} \bbetaks{t}} 
    \leq 
    \opnorm{(\hbarA^\perp)^\top \oA}
    \twonorm{\bbetaks{t}}
    \lesssim 
    \frac{\tau}{\sigmaminin \sqrt{1-\varepsilon}} \zetak{t}.
\end{equation}
Therefore, 
\begin{equation} \label{eq:bias_bound}
\norm{({\matr{P}}_{\Abarhat} - \matr{I}) \bbeta^{(t)*} }_2
\lesssim 
\frac{\zetak{t} \tau}
{\sigma_{\min,\text{in}} \sqrt{1-\varepsilon} } 
+
h \zetak{t}.
\end{equation}

Now we bound the estimation error term
$\norm{\Abarhatthetahatt - \bbeta^{(t)}_{\text{proj}} }_2 
$.
Recall that $\widehat{\matr{\Sigma}}^{(t)} = \frac{1}{n}\matr{X}^{(t)\top}\matr{X}^{(t)}$ is the sample covariance.
And $\thetahatt = (\Abarhat^\top \widehat{\matr{\Sigma}}^{(t)} \Abarhat)^{-1} \frac{1}{n} \Abarhat^\top \matr{X}^{(t)\top} y^{(t)}$.
Substituting $y^{(t)} = \matr{X}^{(t)}\bbeta^{(t)*} + \epsilon^{(t)}$:
\begin{equation}
\thetahatt = (\Abarhat^\top \widehat{\matr{\Sigma}}^{(t)} \Abarhat)^{-1} \Abarhat^\top \widehat{\matr{\Sigma}}^{(t)} \bbeta^{(t)*} + 
(\Abarhat^\top \widehat{\matr{\Sigma}}^{(t)} \Abarhat)^{-1} \frac{1}{n} \Abarhat^\top \matr{X}^{(t)\top} \epsilon^{(t)}.
\end{equation}
The estimation error term is $\norm{\Abarhatthetahatt - {\matr{P}}_{\Abarhat} \bbeta^{(t)*}}_2 = \norm{\Abarhat\thetahatt - \Abarhat\Abarhat^\top \bbeta^{(t)*}}_2$.
Note that
\begin{align}
\thetahatt - \Abarhat^\top \bbeta^{(t)*} &= \left[ (\Abarhat^\top \widehat{\matr{\Sigma}}^{(t)} \Abarhat)^{-1} \Abarhat^\top \widehat{\matr{\Sigma}}^{(t)} - \Abarhat^\top \right] \bbeta^{(t)*} + (\Abarhat^\top \widehat{\matr{\Sigma}}^{(t)} \Abarhat)^{-1} \frac{1}{n} \Abarhat^\top \matr{X}^{(t)\top} \epsilon^{(t)} \\
&= (\Abarhat^\top \widehat{\matr{\Sigma}}^{(t)} \Abarhat)^{-1} \left[ \Abarhat^\top \widehat{\matr{\Sigma}}^{(t)} - \Abarhat^\top \widehat{\matr{\Sigma}}^{(t)} \Abarhat \Abarhat^\top \right] \bbeta^{(t)*} + 
(\Abarhat^\top \widehat{\matr{\Sigma}}^{(t)} \Abarhat)^{-1} \frac{1}{n} \Abarhat^\top \matr{X}^{(t)\top} \epsilon^{(t)} \\
&= (\Abarhat^\top \widehat{\matr{\Sigma}}^{(t)} \Abarhat)^{-1} \Abarhat^\top \widehat{\matr{\Sigma}}^{(t)} (\matr{I} - {\matr{P}}_{\Abarhat}) \bbeta^{(t)*} + (\Abarhat^\top \widehat{\matr{\Sigma}}^{(t)} \Abarhat)^{-1} \frac{1}{n} \Abarhat^\top \matr{X}^{(t)\top} \epsilon^{(t)}.
\end{align}
The estimation error is therefore bounded by:
\begin{equation}
\norm{\Abarhatthetahatt - \bbeta^{(t)}_{\text{proj}}}_2 \le 
\opnorm{\Abarhat(\Abarhat^\top \widehat{\matr{\Sigma}}^{(t)} \Abarhat)^{-1} \Abarhat^\top \widehat{\matr{\Sigma}}^{(t)}} 
\norm{(\matr{I} - {\matr{P}}_{\Abarhat}) \bbeta^{(t)*}}_2 + \norm{(\Abarhat^\top \widehat{\matr{\Sigma}}^{(t)} \Abarhat)^{-1} \frac{1}{n} \Abarhat^\top \matr{X}^{(t)\top} \epsilon^{(t)}}.
\end{equation}
By Lemma \ref{lemma:cov-hat},
the operator norm of the term 
$\opnorm{\Abarhat(\Abarhat^\top \widehat{\matr{\Sigma}}^{(t)} \Abarhat)^{-1} \Abarhat^\top \widehat{\matr{\Sigma}}^{(t)}} $
is bounded by a constant. 
By standard linear regression results, the noise part 
$\norm{(\Abarhat^\top \widehat{\matr{\Sigma}}^{(t)} \Abarhat)^{-1} \frac{1}{n} \Abarhat^\top \matr{X}^{(t)\top} \epsilon^{(t)}}$
is bounded by $\mathcal{O}_{\tP}(\sqrt{\hat{k}/n})$.
Thus,  \wpr,
\begin{equation}
\norm{\Abarhatthetahatt - {\matr{P}}_{\Abarhat} \bbeta^{(t)*}}_2
\lesssim 
\norm{({\matr{P}}_{\Abarhat} - \matr{I}) \bbeta^{(t)*} }_2 + \sqrt{\frac{\hat{k}}{n}}.
\end{equation}
Combining the Bounds, we get, \wpr
\begin{equation}
    \norm{\Abarhatthetahatt - \bbeta^{(t)*} }_2 
    \lesssim 
    \sqrt{\frac{\hat{k}}{n}} + 
    h \zetak{t}
    +
    \frac{\zetak{t} \tau}{\sigma_{\min,\text{in}} \sqrt{1-\varepsilon}} .
\end{equation}
Plugging in $\tau$, we have, \wpr
\begin{equation}
    \norm{\Abarhatthetahatt - \bbeta^{(t)*} }_2 
    \lesssim 
    \sqrt{\frac{\hat{k}}{n}} + 
    h \zetak{t}
    +
    \frac{\zetak{t} }{\sigma_{\min,\text{in}} \sqrt{1-\varepsilon}} 
    \sqrt{\frac{p+\cardS}{nT}}
    +
    \frac{\zetak{t}}{\sigmaminin}
    \barzeta h 
    \bigg[\frac{\sigmamax(\bm{D}^*_S)}{\sqrt{r}\sigmamin(\bm{D}^*_S)} \wedge 1\bigg].
\end{equation}

Define $\eta^{(t)} = 
\sqrt{\frac{\hat{k}}{n}} + 
\sqrt{\frac{\log T}{n}}
 + 
    h \zetak{t}
    +
    \frac{\zetak{t} }{\sigma_{\min,\text{in}} \sqrt{1-\varepsilon}} 
    \sqrt{\frac{p+\cardS}{nT}}
    +
    \frac{\zetak{t}}{\sigmaminin}
    \barzeta h 
    \bigg[\frac{\sigmamax(\bm{D}^*_S)}{\sqrt{r}\sigmamin(\bm{D}^*_S)} \wedge 1\bigg].$
Given $\gamma = C'\sqrt{p+\log T}$ with a large constant $C' > 0$,
for any inlier task $t$ satisfying $\eta^{(t)} \leq C\sqrt{\frac{p + \log T}{n}}$, we have,
\begin{equation}
	\twonorm{\nabla \fk{t}(\bAks{t}\bthetaks{t})} + C\twonorm{
    \Abarhat
    \hthetak{t}-\bAks{t}\bthetaks{t}} \lesssim \frac{\gamma}{\sqrt{n}}.
\end{equation}
By Lemma \ref{lemma:safe net mtl}(\rom{1}), \wpr, we have,
$$
\twonorm{\hbetak{t} - \bbetaks{t}} \lesssim \eta^{(t)}.
$$
For any inlier task $t$ satisfying $\eta^{(t)} > C\sqrt{\frac{p + \log T}{n}}$, by Lemma \ref{lemma:safe net mtl}(\rom{2}), \wppp, 
$$
\max_{t \in S}\twonorm{\hbetak{t} - \bbetaks{t}} \lesssim \frac{\gamma}{\sqrt{n}} + \max_{t \in S}\twonorm{\widetilde{\bbeta}^{(t)} - \bbetaks{t}} \lesssim \sqrt{\frac{p+\log T}{n}}
$$ 
where $\widetilde{\bbeta}^{(t)} \in \argmin_{\bbeta \in \mathbb{R}^p}\fk{t}(\bbeta) $. 

For outlier tasks satisfying the linear model, Lemma \ref{lemma:safe net mtl}(\rom{3}) gives the upper bound 
$$
\max_{t \in S^c}\twonorm{\hbetak{t} - \bbetaks{t}} \lesssim \frac{\gamma}{\sqrt{n}} + \max_{t \in S^c}\twonorm{\widetilde{\bbeta}^{(t)} - \bbetaks{t}} \lesssim \sqrt{\frac{p+\log T}{n}}
$$ 
where $\widetilde{\bbeta}^{(t)} \in \argmin_{\bbeta \in \mathbb{R}^p}\fk{t}(\bbeta) $. 
\end{proof}

\subsection{Proof of Proposition \ref{thm:TL-error}}
\begin{proof}
Let ${\matr{P}}_{\Abarhat} = \Abarhat\Abarhat^\top$ be the projector onto the estimated subspace. Define the projection of the true coefficient vector onto this subspace, $\bbeta^{(0)}_{\text{proj}} = {\matr{P}}_{\Abarhat} \bbeta^{(0)*}$.
Recall that 
$\hthetak{0} = \argmin_{\btheta \in \mathbb{R}^{\hat{k}}}\fk{0}
(\Abarhat \btheta).$
By the triangle inequality, 
\begin{equation}
\norm{\Abarhat \thetahatk{0} - \bbeta^{(0)*} }_2 \le 
\norm{\Abarhat \thetahatk{0} - \bbeta^{(0)}_{\text{proj}} }_2 
+ 
\norm{\bbeta^{(0)}_{\text{proj}} - \bbeta^{(0)*} }_2.
\end{equation}
The approximation error term
\begin{equation}
\norm{\bbeta^{(0)}_{\text{proj}} - \bbeta^{(0)*} }_2 = \norm{({\matr{P}}_{\Abarhat} - \matr{I}) \bbeta^{(0)*} }_2
\leq \twonorm{\bbetaks{0}-\bP_{\oA} \bbetaks{0}}+ 
\twonorm{\bP_{\oA} \bbetaks{0} - \bP_{\hbarA} \bbetaks{0}}.
\end{equation}
By \eqref{eq:delta-t-star-bd}, $\twonorm{\bbetaks{0}-\bP_{\oA} \bbetaks{0}} \leq h_0 \zetak{0}.$
Note that, by Proposition \ref{prop:subspace_error_bd_general},
\begin{equation}
    \twonorm{\bP_{\oA} \bbetaks{0} - \bP_{\hbarA} \bbetaks{0}} 
    \leq 
    \opnorm{(\hbarA^\perp)^\top \oA}
    \twonorm{\bbetaks{0}}
    \lesssim 
    \frac{\tau}{\sigmaminin \sqrt{1-\varepsilon}} \zetak{0}.
\end{equation}
Therefore, 
\begin{equation} 
\norm{({\matr{P}}_{\Abarhat} - \matr{I}) \bbeta^{(0)*} }_2
\lesssim 
\frac{\zetak{0} \tau}
{\sigma_{\min,\text{in}} \sqrt{1-\varepsilon} } 
+
h_0 \zetak{0}.
\end{equation}

Now we bound the estimation error term
$\norm{\Abarhat \thetahatk{0} - \bbeta^{(0)}_{\text{proj}} }_2 
$.
Recall that $\widehat{\matr{\Sigma}}^{(0)} = \frac{1}{n_0}\matr{X}^{(0)\top}\matr{X}^{(0)}$ is the sample covariance.
And $\thetahatk{0} = (\Abarhat^\top \widehat{\matr{\Sigma}}^{(0)} \Abarhat)^{-1} \frac{1}{n_0} \Abarhat^\top \matr{X}^{(0)\top} y^{(0)}$.
Substituting $y^{(0)} = \matr{X}^{(0)}\bbeta^{(0)*} + \epsilon^{(0)}$:
\begin{equation}
\thetahatk{0} = (\Abarhat^\top \widehat{\matr{\Sigma}}^{(0)} \Abarhat)^{-1} \Abarhat^\top \widehat{\matr{\Sigma}}^{(0)} \bbeta^{(0)*} + 
(\Abarhat^\top \widehat{\matr{\Sigma}}^{(0)} \Abarhat)^{-1} \frac{1}{n_0} \Abarhat^\top \matr{X}^{(0)\top} \epsilon^{(0)}.
\end{equation}
The estimation error term is $\norm{\Abarhat \thetahatk{0} - {\matr{P}}_{\Abarhat} \bbeta^{(0)*}}_2 = \norm{\Abarhat\thetahatk{0} - \Abarhat\Abarhat^\top \bbeta^{(0)*}}_2$.
Note that
\begin{align}
\thetahatk{0} - \Abarhat^\top \bbeta^{(0)*} &= \left[ (\Abarhat^\top \widehat{\matr{\Sigma}}^{(0)} \Abarhat)^{-1} \Abarhat^\top \widehat{\matr{\Sigma}}^{(0)} - \Abarhat^\top \right] \bbeta^{(0)*} + (\Abarhat^\top \widehat{\matr{\Sigma}}^{(0)} \Abarhat)^{-1} \frac{1}{n_0} \Abarhat^\top \matr{X}^{(0)\top} \epsilon^{(0)} \\
&= (\Abarhat^\top \widehat{\matr{\Sigma}}^{(0)} \Abarhat)^{-1} \left[ \Abarhat^\top \widehat{\matr{\Sigma}}^{(0)} - \Abarhat^\top \widehat{\matr{\Sigma}}^{(0)} \Abarhat \Abarhat^\top \right] \bbeta^{(0)*} + 
(\Abarhat^\top \widehat{\matr{\Sigma}}^{(0)} \Abarhat)^{-1} \frac{1}{n_0} \Abarhat^\top \matr{X}^{(0)\top} \epsilon^{(0)} \\
&= (\Abarhat^\top \widehat{\matr{\Sigma}}^{(0)} \Abarhat)^{-1} \Abarhat^\top \widehat{\matr{\Sigma}}^{(0)} (\matr{I} - {\matr{P}}_{\Abarhat}) \bbeta^{(0)*} + (\Abarhat^\top \widehat{\matr{\Sigma}}^{(0)} \Abarhat)^{-1} \frac{1}{n_0} \Abarhat^\top \matr{X}^{(0)\top} \epsilon^{(0)}.
\end{align}
The estimation error is therefore bounded by:
\begin{equation}
\norm{\Abarhat \thetahatk{0} - \bbeta^{(0)}_{\text{proj}}}_2 \le 
\opnorm{\Abarhat(\Abarhat^\top \widehat{\matr{\Sigma}}^{(0)} \Abarhat)^{-1} \Abarhat^\top \widehat{\matr{\Sigma}}^{(0)}} 
\norm{(\matr{I} - {\matr{P}}_{\Abarhat}) \bbeta^{(0)*}}_2 + \norm{(\Abarhat^\top \widehat{\matr{\Sigma}}^{(0)} \Abarhat)^{-1} \frac{1}{n_0} \Abarhat^\top \matr{X}^{(0)\top} \epsilon^{(0)}}.
\end{equation}
By Lemma \ref{lemma:cov-hat},
the operator norm of the term 
$\opnorm{\Abarhat(\Abarhat^\top \widehat{\matr{\Sigma}}^{(0)} \Abarhat)^{-1} \Abarhat^\top \widehat{\matr{\Sigma}}^{(0)}} $
is bounded by a constant. 
By standard linear regression results, the noise part 
$\norm{(\Abarhat^\top \widehat{\matr{\Sigma}}^{(0)} \Abarhat)^{-1} \frac{1}{n_0} \Abarhat^\top \matr{X}^{(0)\top} \epsilon^{(0)}}$
is bounded by $\mathcal{O}_{\tP}(\sqrt{\hat{k}/n_0})$.
Thus,  \wpr,
\begin{equation}
\norm{\Abarhat \thetahatk{0} - {\matr{P}}_{\Abarhat} \bbeta^{(0)*}}_2
\lesssim 
\norm{({\matr{P}}_{\Abarhat} - \matr{I}) \bbeta^{(0)*} }_2 + \sqrt{\frac{\hat{k}}{n_0}}.
\end{equation}
Combining the Bounds, we get, \wpr
\begin{equation}
    \norm{\Abarhat \thetahatk{0} - \bbeta^{(0)*} }_2 
    \lesssim 
    \sqrt{\frac{\hat{k}}{n_0}} + 
    h_0 \zetak{0}
    +
    \frac{\zetak{0} \tau}{\sigma_{\min,\text{in}} \sqrt{1-\varepsilon}} .
\end{equation}
Plugging in $\tau$, we have, \wpr
\begin{equation}
    \norm{\Abarhat \thetahatk{0} - \bbeta^{(0)*} }_2 
    \lesssim 
    \sqrt{\frac{\hat{k}}{n_0}} + 
    h_0 \zetak{0}
    +
    \frac{\zetak{0} }{\sigma_{\min,\text{in}} \sqrt{1-\varepsilon}} 
    \sqrt{\frac{p+\cardS}{nT}}
    +
    \frac{\zetak{0}}{\sigmaminin}
    \barzeta h 
    \bigg[\frac{\sigmamax(\bm{D}^*_S)}{\sqrt{r}\sigmamin(\bm{D}^*_S)} \wedge 1\bigg].
\end{equation}

Define $\eta^{(0)} = 
\sqrt{\frac{\hat{k}}{n_0}} + 
    h_0 \zetak{0}
    +
    \frac{\zetak{0} }{\sigma_{\min,\text{in}} \sqrt{1-\varepsilon}} 
    \sqrt{\frac{p+\cardS}{nT}}
    +
    \frac{\zetak{0}}{\sigmaminin}
    \barzeta h 
    \bigg[\frac{\sigmamax(\bm{D}^*_S)}{\sqrt{r}\sigmamin(\bm{D}^*_S)} \wedge 1\bigg].$
Given $\gamma = C'\sqrt{p+\log T}$ with a large constant $C' > 0$,
if $\eta^{(0)} \leq C\sqrt{\frac{p }{n_0}}$, we have,
\begin{equation}
	\twonorm{\nabla \fk{0}(\bAks{0}\bthetaks{0})} + C\twonorm{
    \Abarhat    
        \hthetak{0}-\bAks{0}\bthetaks{0}} \lesssim \frac{\gamma}{\sqrt{n_0}}.
\end{equation}
By Lemma \ref{lemma:TL safe net mtl}(\rom{1}), \wpr, we have,
$$
\twonorm{\hbetak{0} - \bbetaks{0}} \lesssim \eta^{(0)}.
$$
If $\eta^{(0)} > C\sqrt{\frac{p}{n_0}}$, by Lemma \ref{lemma:TL safe net mtl}(\rom{2}), \wppp, 
$$
\twonorm{\hbetak{t} - \bbetaks{0}} \lesssim \frac{\gamma}{\sqrt{n_0}} + \max_{t \in S}\twonorm{\widetilde{\bbeta}^{(0)} - \bbetaks{0}} \lesssim \sqrt{\frac{p}{n_0}}
$$ 
where $\widetilde{\bbeta}^{(0)} \in \argmin_{\bbeta \in \mathbb{R}^p}\fk{0}(\bbeta) $. 
\end{proof}

\section{Supplementary Lemmas}
\begin{lemma}
\label{lemma:noise_opnorm}
    Under Assumption \ref{assump:covariate} and \ref{assump:n}, 
    w.p. at least $1-e^{-C(p+ \log T)}$,
\begin{equation}
    \opnorm{( (\widehat{\bB}_{\textup{st}})_S 
    -\bB^*_S)/\sqrt T}
    \lesssim
    \sqrt{\frac{p+\cardS}{nT}}.
\end{equation}
\end{lemma}
\begin{proof}
    This is a standard result of linear regression. See \citet[Theorem 5.39]{Vershynin2010IntroductionToThe}, or \citet[Appendix D.8]{TianGuFeng2023LearningFromSimilar} for a proof.
\end{proof}

\begin{lemma}[Theorem 6.5 in \cite{wainwright2019high}, Lemma 18 in \citet{TianGuFeng2023LearningFromSimilar}]
\label{lemma:cov-hat}
	Under Assumptions \ref{assump:covariate} and \ref{assump:n}, for any $\delta > 0$ and any $t \in [T]$, w.p. at least $1-C_1e^{-nC_2(\delta\wedge \delta^2)}$,
	\begin{equation}
		\twonorm{\hSigmak{t}-\bSigmak{t}} \leq C_3\sqrt{\frac{p}{n}} + \delta,
	\end{equation}
	with some universal constants $C_1, C_2, C_3 > 0$. 
    As a corollary, w.p. at least $1-e^{-C(p+\log T)}$,
	\begin{equation}
		\max_{t\in [T]}\twonorm{\hSigmak{t}-\bSigmak{t}} \lesssim \sqrt{\frac{p+\log T}{n}}.
	\end{equation}
\end{lemma}

\begin{lemma}[Lemma 38 in \citet{TianGuFeng2023LearningFromSimilar}]
\label{lemma:safe net mtl}
	Under Assumptions \ref{assump:covariate}-\ref{assump:n}, 
	\begin{enumerate}[(i)]
		\item For all $t \in S$, when $\frac{\gamma}{\sqrt{n}} \geq \twonorm{\nabla \fk{t}(\bAks{t}\bthetaks{t})} + C\twonorm{
        \Abarhat
        \hthetak{t} - \bAks{t}\bthetaks{t}}$, it holds that $\hbetak{t} = 
        \Abarhat
        \hthetak{t}$ \wpp;
		\item $\max_{t \in S}\twonorm{\hbetak{t} - \bbetaks{t}} \leq C\frac{\gamma}{\sqrt{n}} + \max_{t \in S}\twonorm{\widetilde{\bbeta}^{(t)} - \bbetaks{t}}$ \wpp, where $\widetilde{\bbeta}^{(t)} \in \argmin_{\bbeta \in \mathbb{R}^p}\fk{t}(\bbeta)$;
		\item If the data from tasks in $S^c$ satisfies the linear model 
        and Assumption \ref{assump:covariate}, then \wpp, $\max_{t \in S^c}\twonorm{\hbetak{t} - \bbetaks{t}} \leq C\frac{\gamma}{\sqrt{n}} + \max_{t \in S^c}\twonorm{\widetilde{\bbeta}^{(t)} - \bbetaks{t}}$, where $\widetilde{\bbeta}^{(t)} \in \argmin_{\bbeta \in \mathbb{R}^p}\fk{t}(\bbeta)$.
	\end{enumerate}
\end{lemma}

\begin{lemma}[Lemma 41 in \citet{TianGuFeng2023LearningFromSimilar}]\label{lemma:TL safe net mtl}
	Under Assumptions \ref{assump:TLcovariate}-\ref{assump:TLn}, we have:
	\begin{enumerate}[(i)]
		\item When $\frac{\gamma}{\sqrt{n_0}} \geq \twonorm{\nabla \fk{0}(\bAks{0}\bthetaks{0})} + C\twonorm{\hSigmak{0}}\cdot
    \twonorm{\hbarA
    \widehat{\btheta}^{(0)}
    - \bAks{0}\bthetaks{0}}$, it holds that $\hbetak{0} = \hbarA
    \widehat{\btheta}^{(0)}
    $ \wpp; 
		\item $\twonorm{\hbetak{0} - \bbetaks{0}} \leq C\frac{\gamma}{\sqrt{n_0}} + \twonorm{\widetilde{\bbeta}^{(0)} - \bbetaks{0}}$ \wpp, where $\widetilde{\bbeta}^{(0)} \in \argmin_{\bbeta \in \mathbb{R}^p}\fk{0}(\bbeta)$.
	\end{enumerate}
\end{lemma}

\end{document}